\documentclass{article}

\usepackage{amsmath,amsthm,amssymb,dsfont}
\usepackage{bbm,multirow}
\allowdisplaybreaks
\usepackage{tikz}
\usepackage{authblk}
\usepackage{times}
\usepackage{fullpage}
\usepackage{enumitem}
\usepackage{bm}
\usepackage{commath}
\usepackage[numbers]{natbib}
\usepackage{verbatim}
\usepackage{multicol}
\usepackage{caption}

\usepackage[utf8]{inputenc} % Required for inputting international characters
\usepackage[T5]{fontenc} % Output font encoding for international characters

\usepackage{algorithm}
\usepackage{algorithmic}

\newtheorem{theorem}{Theorem}
\newtheorem{lemma}{Lemma}

\newtheorem{corollary}{Corollary}

\newcommand{\cmt}[1]{}

\newcommand{\vect}[1]{\ensuremath{\bm{#1}}}
\newcommand{\E}{\ensuremath{\mathbb{E}}}

\title{Non-monotone DR-submodular Maximization: Approximation and Regret Guarantees
\thanks{Research supported by the ANR project OATA n\textsuperscript{o} ANR-15-CE40-0015-01}}

\author[1]{Christoph D\"urr
}
\author[2]{Nguyễn Kim Thắng
}
\author[2]{Abhinav Srivastav
}
\author[2]{Léo Tible
}
\affil[1]{LIP6, Sorbonne Universit\'e \\}
\affil[2]{IBISC, Univ \'Evry, Universit\'e Paris-Saclay\\}
\begin{document}

%\date{\today}
\maketitle 

\begin{abstract}

%Many optimization problems can often be cast as the optimization of a set-function defined over a convex domain. Among these set-functions, submodular functions pay an important role with applications in many areas of applied mathematics, such as machine learning, computer vision, operation research, communication systems or economics. However unlike set-functions, the diminishing returns (DR) property of submodular functions does not hold for the continuous domain. Here, the DR property defines a subclass of submodular functions, called DR-submodular functions that are known to have both, theoretical guarantees and practical applications. 

Diminishing-returns (DR) submodular optimization is an important field with many real-world applications in machine learning, economics and communication systems. It captures a subclass of non-convex optimization that provides both practical and theoretical guarantees.

In this paper, we study the fundamental problem of maximizing non-monotone DR-submodular functions over down-closed and general convex sets in both offline and online settings.  First, we show that for offline maximizing non-monotone DR-submodular functions over a general convex set, the Frank-Wolfe algorithm achieves an approximation guarantee which depends on the
convex set. Next, we show that the Stochastic Gradient Ascent algorithm achieves a 1/4-approximation ratio with 
the regret  of $O(1/\sqrt{T})$ for the problem of maximizing non-monotone DR-submodular functions over down-closed convex 
sets. These are the first approximation guarantees in the corresponding settings. Finally we benchmark these algorithms on problems arising in machine learning domain with the real-world datasets.

\end{abstract}

\section{Introduction}

%We consider the quite general problem of optimizing some function $F$ over a convex set $\cal K$, which by normalization we assume to be included in $[0,1]^n$.  The problem has recently gained great attention for \emph{diminishing returns submodular} functions \cite{Bach16:Submodular-functions:,BianBuhmann18:Optimal-DR-Submodular,BianLevy17:Non-monotone-Continuous,ChenHassani18:Online-continuous,HassaniSoltanolkotabi17:Gradient-methods,NiazadehRoughgarden18:Optimal-Algorithms,StaibJegelka17:Robust-Budget}. This property of functions, also called DR-submodularity, has applications in machine learning and statistics such as non-definite quadratic programming \cite{ItoFujimaki16:Large-scale-price}, determinantal point processes \cite{KuleszaTaskar12:Determinantal-point}, log-submodular models \cite{DjolongaKrause14:From-MAP-to-marginals:}, etc.

We consider the fundamental problem of optimizing DR-submodular function over a convex set. This problem has recently gained a significant attention in both, machine learning and theoretical computer science communities~\cite{Bach16:Submodular-functions:,BianBuhmann18:Optimal-DR-Submodular,BianLevy17:Non-monotone-Continuous,ChenHassani18:Online-continuous,HassaniSoltanolkotabi17:Gradient-methods,NiazadehRoughgarden18:Optimal-Algorithms,StaibJegelka17:Robust-Budget} due to its numerous applications in formulating real-world problems. Some examples of this can be found in~\cite{ItoFujimaki16:Large-scale-price}, \cite{KuleszaTaskar12:Determinantal-point}, and \cite{DjolongaKrause14:From-MAP-to-marginals:}.

Previous work on this problem have been focused either on smooth and/or monotone DR-submodular functions or, on unconstrained or down-closed convex sets. 
Though, the majority of real-world problems can be formulated as non-monotone DR-submodular functions over a constrained convex set. 
Hence in this paper, we investigate the problem of maximizing constrained non-monotone DR-submodular functions. 
Our contribution is twofold.  First, we provide an approximation algorithm for maximizing smooth non-monotone DR-submodular function over general convex sets. 
%which is a setting that has only be studied under additional constraints, either monotonicity of $F$, or down-closedness of $\cal K$.  
Second we provide an online algorithm for maximizing non-monotone DR-submodular function over down-closed convex sets.  
Prior to this work, no theoretical guarantees were known for both of these problems.
%Again this is a setting which has only been studied previously for monotone functions, however without requiring down-closedness of the search space.

%\subsection{The setting}

Without the loss of generality, we assume that the DR-submodular function $F$ is positive, and that at any point $\vect{x}$ in the convex set $\cal K$, $F(\vect{x})$ and its gradient (denoted by $\nabla F(\vect{x})$) can be evaluated in the polynomial time.  In addition, we assume that the projection of any point $x\in[0,1]^n$ on $\cal K$ can be computed in polynomial time (this implies the availability of a polynomial time membership oracle).

	Our offline algorithm is a discrete time local search procedure which produces a solution $\vect{x}^{T}$ after
	$T$ iterations, such that
	$$
		F(\vect{x}^{T}) \geq \max_{\vect{x}^{*} \in \mathcal{K}} \alpha \cdot F(\vect{x}^{*}) - \beta
	$$
	where $\alpha, \beta$ are some parameters. Such an algorithm is called an $\alpha$-\emph{approximation} with
	\emph{convergence rate} $\beta$.

The online setting consists of discrete time steps $t=1,\ldots, T$ for some time horizon $T$.  At each step $t$ the algorithm first outputs a point $\vect{x}^t \in {\cal K}$, and then learns  a function $F^t : {\cal K} \rightarrow {\mathbb R}^+$.  The value $F^t(\vect{x}^t)$ is called its reward and the goal of the algorithm is to maximize the average reward. Hence, the goal is to minimize the regret. Formally we say that an algorithm achieves $(\alpha, \beta)$-\emph{regret} if it produces points $\vect{x}^{t}$ such that
	\[
		 \frac{1}{T} \sum_{t=1}^T F^t(\vect{x}^t)
		 	\geq \alpha \cdot \max_{\vect{x}^{*}\in \cal K} \frac{1}{T} \sum_{t=1}^T F^t(\vect{x}^{*}) - \beta.
	\]
	Equivalently, we say that the algorithm has $\alpha$-\emph{regret} at most $\beta$. The factor $\alpha$ is also called the
	\emph{approximation ratio} of the algorithm.

\subsection{Our contributions}
Exploring the underlying properties DR-submodularity, we design algorithms with performance guarantees for each of the above mentioned settings.  Our contributions are summarized as follows: (also see Table~\ref{table:DR-submodular}).
\newcommand{\STAB}[1]{\begin{tabular}{@{}c@{}}#1\end{tabular}}
\renewcommand{\arraystretch}{1.5}

\begin{table}[th]
  \small
  \centering
  \begin{tabular}{|c|c||c|c|c|c|}
    \hline
     \multicolumn{2}{|c||}{} & \multicolumn{2}{|c|}{Monotone} & \multicolumn{2}{|c|}{Non-monotone}\\
\cline{3-6}
     \multicolumn{2}{|c||}{} & smooth  & non-smooth & smooth  & non-smooth \\
    \hline
    \hline
    \multirow{3}{*}{\STAB{\rotatebox[origin=c]{90}{Offline}}}& unconstrained &  & & & $\frac{1}{2}$-approx \\
    		& &  & & &  \multicolumn{1}{c|}{\cite{BianBuhmann18:Optimal-DR-Submodular,NiazadehRoughgarden18:Optimal-Algorithms}} \\
    		\cline{2-6}
    		 & down-closed & $\left( 1 - \frac{1}{e}, O\bigl(\frac{1}{\sqrt{T}}\bigr) \right)$  & & $\left(\frac{1}{e}, O\bigl(\frac{1}{T}\bigr) \right)$   &  \\
		 &  & \cite{BianMirzasoleiman17:Guaranteed-Non-convex} & &  \cite{BianLevy17:Non-monotone-Continuous}  &  \\
		 \cline{2-6}
		 & general  &  & $\left( \frac{1}{2}, O\bigl(\frac{1}{\sqrt{T}}\bigr) \right)$ 
		 &\color{red}{$\left(  \frac{1- \min \limits_{\vect{x} \in \mathcal{K}} \|\vect{x}\|_{\infty} }{3\sqrt{3}}, O\bigl(\frac{1}{\ln^{2}T}\bigr) \right)$}&   \\
		 &   &  & \cite{HassaniSoltanolkotabi17:Gradient-methods}
		 &&  \\    \hline
    \hline
     \multirow{3}{*}{\STAB{\rotatebox[origin=c]{90}{Online}}}& unconstrained & & & &  \\
    		\cline{2-6}
    		 & down-closed  & & &
    		 &\color{red}{$ \left( \frac{1}{4}, O\bigl(\frac{1}{\sqrt{T}}\bigr) \right)$} \\
		 \cline{2-6}
		 & general  & $\left( 1 - \frac{1}{e}, O\bigl(\frac{1}{\sqrt{T}}\bigr) \right)$   & $\left( \frac{1}{2}, O\bigl(\frac{1}{\sqrt{T}}\bigr) \right)$  &   &  \\
		 &   &  \cite{ChenHassani18:Online-continuous}  &  \cite{ChenHassani18:Online-continuous} &   &  \\
    \hline
  \end{tabular}

  \caption{Summary of results on DR-submodular maximization. Results from the present paper are shown in red. Entries $(\alpha,\beta)$ refer either to an $\alpha$-approximation offline algorithm with convergence ratio $\beta$ or to an $(\alpha, \beta)$-regret online algorithm.
	%xtof: that's not true for offline, general, non-monotone, non-smooth
	% The guarantee in a blank cell
	% can be deduced from a more general one.
	}
  \label{table:DR-submodular}
\end{table}

\begin{description}
	\item[Offline setting.]
		First, we consider the problem of maximizing a non-monotone DR-submodular function over general convex sets. 
 		This problem has been proved to be hard. Specifically, any constant-approximation algorithm for the problem
		over a general convex domain must require exponentially many value queries
		to the function \cite{Vondrak13:Symmetry-and-approximability}. Determining the approximation ratio
		as a function depending on the problem parameters and characterizing necessary and sufficient
		regularity conditions that enable efficient approximation algorithm for the problem constitute
		an important direction.

		We show that the celebrated Frank-Wolf algorithm achieves
		an approximation ratio of $\left(\frac{1-\min_{\vect{x} \in \mathcal{K}} \|\vect{x}\|_{\infty} }{3\sqrt{3}}\right)$
		with the rate of convergence of $O\bigl(\frac{n}{\ln^{2}T}\bigr)$, where
		$T$ is the number of iterations applied in the algorithm.
		In particular, if the domain $\cal K$
		(not necessarily down-closed\footnote{We refer to Section~\ref{sec:pre} for a formal definition})
		contains the origin $\vect{0}$ then for arbitrary constant $\epsilon > 0$,
		after $T = O\bigl(e^{\sqrt{n/\epsilon}}\bigr)$ (sub-exponential)
		iterations, the algorithm outputs a solution $\vect{x}^{T}$
		such that $F(\vect{x}^{T}) \geq \left(\frac{1}{3\sqrt{3}} \right) \max_{\vect{x}^{*} \in \mathcal{K}} F(\vect{x}^{*}) - \epsilon$.
		To the best of our knowledge, this is the first algorithm with an approximation guarantee for maximizing non-monotone 
		DR-submodular function overs general convex sets.
	\item[Online setting.]
		DR-submodular maximization has been studied in online environments but only for monotone functions.
		However, in numerous applications the functions are intrinsically non-monotone DR-submodular.
		The quest of algorithms with performance guarantee for online non-monotone DR-submodular maximization
		is a major research line.

		We show that the Online Gradient Ascent algorithm achieves $(1/4, O(1/\sqrt{T}))$-regret.
		The result holds also if only unbiased estimates of the gradients are available.
		Prior to our work, no approximation guarantee has been shown even for the simpler setting of maximizing 
		online non-monotone DR-submodular functions over an unconstrained hypercube (i.e., $\mathcal{K} = [0,1]^{n}$).
	%
	% \item[Bandit setting.]
	% 	Submodular optimization has been studied in bandit environments
	% 	only in the context of minimization  \cite{HazanKale12:Online-submodular}.
	% 	The main issue in maximization is that no algorithm with performance guarantee is available
	% 	even in the online setting.
	% 	Building on our result in the online setting, the bandit convex algorithm of \citet{FlaxmanKalai05:Online-convex}
	% 	and exploring the structure of
	% 	the multilinear relaxations of submodular functions,
	% 	we derive an algorithm which achieves $(1/4,O(1/T^{1/4}))$-regret.
	% 	To the best of our knowledge, this is the first approximation guarantee
	% 	for bandit submodular maximization.
	%
	\item[Experiments.]
		We experimentally demonstrate the efficiency of our algorithms on the problems arising in domain of  machine learning. We conduct following three set of experiments. 
		\begin{enumerate} 
			\item We compare the performance of offline Gradient Ascent algorithm against the previous known algorithms for maximizing DR-submodular function over \emph{down-closed} polytopes. Note that our theoretical guarantee holds for more general case. 
			\item We show the performance of offline Gradient Ascent algorithm for revenue maximizing problem on the real-world dataset (Advogato user-user relationship graph) over a \emph{general} (not down-closed)  polytope. 
			\item We show the performance of online Frank-Wolfe algorithm for revenue maximizing revenue on the real-world dataset (Facebook user-user relationship graph) on a  \emph{down-closed} polytope. 
		\end{enumerate}
		%xtof: ?
		% We observe that our algorithms are competitive, even compared to prior ones on the restricted settings where
		% the latter admits theoretical guarantees.
\end{description}

%%%%%%%%%% *****************
%%%%%%%%%% *****************

\subsection{Related work}

% summary about offline (DR)-submodular minimization and maximization.
Submodular optimization has been widely studied for decades \cite{NemhauserWolsey78:An-analysis-of-approximations,Fujishige05:Submodular-functions}. The domain has been investigating even more extensively
in recent years due to numerous applications in statistics and machine learning, for example
active learning \cite{GolovinKrause11:Adaptive-submodularity:},
viral marketing \cite{KempeKleinberg03:Maximizing-the-spread},
network monitoring \cite{Gomez-RodriguezLeskovec10:Inferring-networks},
document summarization \cite{LinBilmes11:A-class-of-submodular},
crowd teaching \cite{SinglaBogunovic14:Near-Optimally-Teaching},
feature selection \cite{ElenbergKhanna18:Restricted-strong},
deep neural networks \cite{ElenbergDimakis17:Streaming-weak},
diversity models \cite{DjolongaTschiatschek16:Variational-inference} and
recommender systems \cite{GuilloryBilmes11:Simultaneous-Learning}.

\paragraph{Offline submodular/DR-submodular optimization.}
The problem of submodular (set) minimization has been studied in \cite{Schrijver00:A-combinatorial-algorithm,IwataFleischer01:A-combinatorial-strongly}.
See \cite{Bachothers13:Learning-with} for a survey on connections with and applications in machine learning.
Submodular (set) maximization is an NP-hard problem. Several approximation algorithms have been given in
the offline setting, for example a 1/2-approximation for unconstrained domains
\cite{BuchbinderFeldman15:A-tight-linear,BuchbinderFeldman18:Deterministic-algorithms},
a $(1-1/e)$-approximation for monotone smooth submodular functions \cite{CalinescuChekuri11:Maximizing-a-monotone,ChekuriJayram15:On-multiplicative-weight},
or a $(1/e)$-approximation for non-motonotone submodular functions on down-closed polytopes \cite{FeldmanNaor11:A-unified-continuous,ChekuriJayram15:On-multiplicative-weight}.

Continuous extension of submodular functions play a crucial role in submodular optimization, especially in submodular
maximization, including the multilinear relaxation and the softmax extension. These belong to the class of
DR-submodular functions. \citet{BianMirzasoleiman17:Guaranteed-Non-convex} considered the
problem of maximizing monotone DR-functions subject to down-closed convex domains and proved that
the greedy method proposed by \cite{CalinescuChekuri11:Maximizing-a-monotone}, which is a variant of the Frank-Wolfe algorithm, guarantees a $(1-1/e)$-approximation.
It has been observed by \citet{HassaniSoltanolkotabi17:Gradient-methods}
that the greedy method is not robust in stochastic settings
(where only unbiased estimates of gradients are available). Subsequently, they
showed that the gradient methods achieve $1/2$-approximations in stochastic settings.
Maximizing non-monotone DR-submodular functions is harder. Very recently,
\citet{BianBuhmann18:Optimal-DR-Submodular} and \citet{NiazadehRoughgarden18:Optimal-Algorithms}
have independently presented algorithms with the same approximation guarantee of 1/2 for the problem of
maximizing non-monotone DR-submodular functions over a hypercube.
Both algorithm are inspired by the bi-greedy algorithm in \cite{BuchbinderFeldman15:A-tight-linear,BuchbinderFeldman18:Deterministic-algorithms}.
\citet{BianLevy17:Non-monotone-Continuous} made a further step by providing an
$1/e$-approximation algorithm over down-closed convex sets.
Remark that when aiming for approximation algorithms, the restriction to down-closed polytopes is unavoidable.
Specifically, \citet{Vondrak13:Symmetry-and-approximability}
proved that any algorithm for the problem over a non-down-closed domain
that guarantees a constant approximation must require exponentially many value queries to the function.

% summary about online (DR)-submodular minimization and maximization.
\paragraph{Online submodular/DR-submodular optimization.}
Online optimization has been broadly studied for convex/concave functions \cite{Hazanothers16:Introduction-to-online}.
An important research agenda is to design algorithms with performance guarantees
in terms of regret and approximation for non-convex functions in general and for DR-submodular functions
in particular. \citet{ChenHassani18:Online-continuous}
have considered the online problem of maximizing monotone DR-submodular functions
and provided an $(1-1/e)$-approximation with regret $O(\sqrt{T})$ when the functions are smooth.
More generally, if the functions are not necessarily smooth, they proved that the online gradient ascent algorithm
achieves a 1/2-approximation with regret $O(\sqrt{T})$. No guarantee has been shown for online
maximizing non-monotone DR-submodular functions. Very recently,
\citet{RoughgardenWang18:An-Optimal-Algorithm} have studied the online problem of
maximizing submodular (set) functions over the unconstrained domain $[0,1]^n$.
They gave an optimal $(1/2, O(\sqrt{T}))$-regret algorithm.

% summary about bandit (DR)-submodular minimization and maximization.
% \paragraph{Bandit submodular optimization.}
% Submodular optimization has been mostly studied for minimization problems \cite{HazanKale12:Online-submodular}.
% To the best of our knowledge, no approximation guarantee for non-monotone
% submodular maximization is proved in the bandit settings.

%%%%%%%%%% *****************
%%%%%%%%%% *****************
%%%%%%%%%% *****************

%%%%%%%%%% *****************
%%%%%%%%%% *****************
%%%%%%%%%% *****************

\section{Preliminaries and Notations}		\label{sec:pre}
We introduce some basic notions, concepts and lemmas which will be used throughout the paper.
We use boldface letters, e.g., $\vect{x}, \vect{z}$ to represent vectors.
%In particular, we use $\vect{x}$ for vectors in $\{0,1\}^{n}$ and $\vect{z}$ for vectors in $[0,1]^{n}$.
We denote $x_{i}$ as the $i^{\text{th}}$ entry of $\vect{x}$ and $\vect{x}^{t}$ as the decision vector at time step $t$.
For two n-dimensional vectors $\vect{x}, \vect{y}$, we say that $\vect{x} \leq \vect{y}$ iff $x_{i} \leq y_{i}$ for all $1 \leq i \leq n$.
Moreover, $\vect{x} \vee \vect{y}$ is defined as a vector such that $(x \vee y)_{i} = \max\{x_{i}, y_{i}\}$
and similarly $\vect{x} \wedge \vect{y}$ is a vector such that $(x \wedge y)_{i} = \min\{x_{i}, y_{i}\}$.
In the paper, we use the Euclidean norm $\|\cdot\|$ by default (so the dual norm is itself).
The infinity norm $\|\cdot\|_{\infty}$ is defined as $\|\vect{x}\|_{\infty} = \max \limits_{i=1}^{n} |{x_{i}}|$.

In the paper, we consider a bounded convex domain $\mathcal{K}$ and w.l.o.g. assume that
$\mathcal{K} \subseteq [0,1]^{n}$. We say that $\mathcal{K}$ is \emph{unconstrained} if $\mathcal{K} = [0,1]^{n}$;
and $\mathcal{K}$ is \emph{down-closed} if for every $\vect{z} \in \mathcal{K}$ and $\vect{y} \leq \vect{z}$
then $\vect{y} \in \mathcal{K}$; and $\mathcal{K}$ is \emph{general} if $\mathcal{K}$ is simply a convex
domain without any particular property. Besides, the \emph{diameter} of the convex domain $\mathcal{K}$ (denoted by $D$) is
defined as $\sup_{\vect{x}, \vect{y} \in \mathcal{K}} \norm{\vect{x} - \vect{y}}$. The
\emph{projection} of a point $\vect{x}$ onto a convex set $\mathcal{K}$ is a
point in $\mathcal{K}$ that is closest to $\vect{x}$; formally defined as follows.
\begin{align}
	\textrm{Proj}_{\mathcal{K}}(\vect{x}) := \arg \min\limits_{\vect{y} \in \mathcal{K}} \norm{\vect{y} - \vect{x}}
\end{align}
A useful property of projections is that they satisfy the Pythagorean inequality, that is for any $\vect{z} \in \mathcal{K}$
and for any $\vect{x}$,
$$
\norm{\textrm{Proj}_{\mathcal{K}}(\vect{x}) - \vect{z}} \leq \norm{\vect{x} - \vect{z}}.
$$

A function $f: \{0,1\}^{n} \rightarrow \mathbb{R}^{+}$ is \emph{submodular} if
for all $\vect{x} \geq \vect{y} \in \{0,1\}^{n}$,
\begin{align}	\label{def:sub}
	f(\vect{x} \vee \vect{a}) - f(\vect{x}) \leq f(\vect{y} \vee \vect{a}) - f(\vect{y})
		\qquad \forall \vect{a} \in \{0,1\}^{n}.
\end{align}
Submodular functions can be generalized over continuous domains. A function
$F: [0,1]^{n} \rightarrow \mathbb{R}^{+}$ is \emph{DR-submodular} if for all vectors
$\vect{x},\vect{y} \in [0,1]^{n}$ with $\vect{x}  \geq \vect{y}$, any basis vector $\vect{e}_{i} = (0,\ldots,0,1,0,\ldots,0)$ and any constant
$\alpha > 0$ such that
$\vect{x} + \alpha \vect{e}_{i} \in [0,1]^{n}$, $\vect{y} + \alpha \vect{e}_{i} \in [0,1]^{n}$,
it holds that
\begin{align}	\label{def:DR-sub}
	F(\vect{x} + \alpha \vect{e}_{i}) - F(\vect{x}) \leq F(\vect{y} + \alpha \vect{e}_{i}) - F(\vect{y}).
\end{align}
Note that if function $F$ is differentiable then the diminishing-return (DR) property (\ref{def:DR-sub}) is equivalent to
\begin{align}	\label{def:DR-sub-diff}
\nabla F(\vect{x}) \leq \nabla F(\vect{y}) \qquad \forall \vect{x} \geq \vect{y} \in [0,1]^{n}.
\end{align}
Moreover, if $F$ is twice-differentiable then
the DR property is equivalent to all of the entries of its Hessian being non-positive, i.e.,
$\frac{\partial^{2} F}{\partial x_{i} \partial x_{j}} \leq 0$ for all $1 \leq i, j \leq n$.
A differentiable function $F: \mathcal{K} \subseteq \mathbb{R}^n \rightarrow \mathbb{R}$ is said to be \textit{$\beta$-smooth} if for any $\vect{x}, \vect{y} \in \mathcal{K}$, we have
\begin{align}	\label{ineq:smooth1}
F(\vect{y}) \leq F(\vect{x}) + \langle \nabla F(\vect{x}), \vect{y} - \vect{x} \rangle + \frac{\beta}{2} \|\vect{y} - \vect{x}\|^{2}
\end{align}
or equivalently,
\begin{align}	\label{ineq:smooth2}
	\norm{\nabla F(\vect{x}) - \nabla F(\vect{y})} \leq \beta \norm{\vect{x} - \vect{y}}.
\end{align}

\subsection*{Properties of DR-submodularity}
In the following, we present properties of DR-submodular functions that are are crucial in our analyses.
The properties have been proved in \cite{HassaniSoltanolkotabi17:Gradient-methods} and \cite{BianLevy17:Non-monotone-Continuous}. For completeness, we provide their proofs in the appendix.

\begin{lemma}[\cite{HassaniSoltanolkotabi17:Gradient-methods}] 	\label{lem:prop2}
For every $\vect{x}, \vect{y} \in \mathcal{K}$ and any DR-submodular function $F: [0,1]^{n} \rightarrow \mathbb{R}^{+}$,
it holds that
$$ \langle \nabla F(\vect{x}), \vect{y} - \vect{x}\rangle \geq F(\vect{x} \vee \vect{y}) + F(\vect{x} \wedge \vect{y}) - 2F(\vect{x}). $$
\end{lemma}

\begin{lemma}[\cite{BianLevy17:Non-monotone-Continuous}]	\label{lem:prop1}
For any DR-submodular function $F$ and for all $\vect{x}, \vect{y}, \vect{z} \in \mathcal{K}$ it holds that
$$F(\vect{x} \vee \vect{y}) + F(\vect{x} \wedge \vect{y}) + F(\vect{z}^* \vee \vect{z}) + F(\vect{z}^* \wedge \vect{z}) \geq F(\vect{y})$$
where $\vect{z}^* = (\vect{x} \vee \vect{y}) - \vect{x}$.
\end{lemma}

\section{Offline Continuous DR-Submodular Maximization}
\label{offline-dr-submod}

In this section, we consider the problem of maximizing a DR-submodular function over a general convex set in the offline setting.
Approximation algorithms \cite{BianMirzasoleiman17:Guaranteed-Non-convex,BianLevy17:Non-monotone-Continuous} have been presented and all of them are adapted variants of the Frank-Wolfe method.
However, those algorithms require that the convex set is down-closed. This structure is crucial in their analyses in order to relate their solution to the optimal solution. Using some property of DR-submodularity (specifically, Lemma~\ref{lem:prop2}),
we show that beyond the down-closed structure, the Frank-Wolf algorithm guarantees an approximation
solution for general convex sets. Below, we present the pseudocode of our variant of the Frank-Wolfe algorithm.

\begin{algorithm}[ht]
\begin{algorithmic}[1]
\STATE Let $\vect{x}^{1} \gets \arg \min_{\vect{x} \in \mathcal{K}} \|\vect{x}\|_{\infty}$.
\FOR{$t = 1$ to $T$}
		%\STATE Let $\mathcal{K}^{t} := \mathcal{K} - \vect{x}^{t-1} = \{\vect{a} - \vect{x}^{t-1}: \vect{a} \in \mathcal{K} \}$
		%\STATE Compute $\vect{v}^{t} \gets \arg \max_{\vect{v} \in \mathcal{K}^{t}} \langle \nabla F(\vect{x}^{t-1}), \vect{v} \rangle$
		%\STATE Update $\vect{x}^{t} \gets \vect{x}^{t-1} + \eta_{t} \vect{v}^{t}$ with step-size
			%$\eta_{t} = \frac{\delta}{t H_{T}}$ where $H_{T}$ is the Harmonic number and $\delta$ is a constant to be defined later.
			%
		\STATE Compute $\vect{v}^{t} \gets \arg \max_{\vect{v} \in \mathcal{K}} \langle \nabla F(\vect{x}^{t-1}), \vect{v} \rangle$
		\STATE Update $\vect{x}^{t} \gets (1 - \eta_t) \vect{x}^{t-1} + \eta_{t} \vect{v}^{t}$ with step-size	$\eta_{t} = \frac{\delta}{t H_{T}}$ where $H_{T}$ is the $T$-th harmonic number and $\delta$ represents the constant $(\ln 3)/ 2$.
\ENDFOR
\RETURN $x^{T}$ %$\arg \max \{F(\vect{x}^{1}), F(\vect{x}^{2}), \ldots, F(\vect{x}^{T})\}$.
\end{algorithmic}
\caption{Frank-Wolfe Algorithm}
\label{algo:offline}
\end{algorithm}

Next, we show that during the execution of the algorithm, the following invariant is maintained.

\begin{lemma}		\label{lem:bound-x}
It holds that $1 - x^{t}_{i} \geq e^{- \delta (1 + O(1/\ln^{2}T))} \cdot (1 - x^{1}_{i})$
for every $1 \leq i \leq n$ and every $1 \leq t \leq T$.
\end{lemma}

\begin{proof}
Fix a dimension $i \in [n]$. We first obtain the following recursion on fixed $x_i$.
\begin{align*}
1 - x^{t}_{i} &= 1 - (1 - \eta_t) x^{t-1}_{i} - \eta_{t} v^{t}_{i}  &\tag{Using the Update step from Algorithm~\ref{algo:offline}} \\
&\geq  1 - (1 - \eta_t)  x^{t-1}_{i} - \eta_{t}  &\tag{$1 \geq v^t_i$}\\
&= (1- \eta_{t}) (1 - x^{t-1}_{i}) \\
&\geq e^{-\eta_{t} - \eta_{t}^{2}} \cdot (1 - x^{t-1}_{i}) &\tag{$1 - u \geq e^{-u - u^{2}}$ for $0 \leq u < 1/2$}
\end{align*}
Using this recursion, we have,
\begin{align*}
1 - x^{t}_{i} &\geq e^{-\sum_{t'=2}^{t} \eta_{t'} -\sum_{t'=2}^{t} \eta_{t'}^{2}} \cdot (1 - x^{1}_{i}) \\
	&\geq e^{-\delta(1 + O(1/\ln^{2}T))} \cdot (1 - x^{1}_{i})
\end{align*}
since $\sum_{t'=2}^{t} \eta_{t'}^{2} = \sum_{t'=2}^{t} \frac{\delta^{2}}{t^{2} H_{T}^{2}} = O(1/\ln^{2}T)$.
\end{proof}

The following lemma was first observed in \cite{FeldmanNaor11:A-unified-continuous} and was generalized
in \cite[Lemma 7]{ChekuriJayram15:On-multiplicative-weight} and \cite[Lemma 3]{BianLevy17:Non-monotone-Continuous}.

\begin{lemma}[\cite{FeldmanNaor11:A-unified-continuous,ChekuriJayram15:On-multiplicative-weight,BianLevy17:Non-monotone-Continuous}]		\label{lem:obs2}
For every $\vect{x}, \vect{y} \in \mathcal{K}$, it holds that $F(\vect{x} \vee \vect{y}) \geq \bigl( 1 - \|\vect{x} \|_{\infty} \bigr) F(\vect{y})$.
\end{lemma}
\medskip 

\begin{theorem}  \label{offline-thm}
Let $\mathcal{K} \subseteq [0,1]^n$ be a convex set and let $F: \mathcal{K} \rightarrow \mathbb{R}$ is a non-monotone $\beta$-smooth DR-submodular function. Let $D$ be the diameter of $\mathcal{K}$. Then Algorithm \ref{algo:offline} yields a solution $\vect{x}^{T} \in \mathcal{K}$ such that the following inequality holds: 
$$
F\bigl(\vect{x}^{T}\bigr)
	\geq
	\left(\frac{1}{3\sqrt{3}} \right) \bigl(1 - \min_{\vect{x} \in \mathcal{K}} \norm{\vect{x}}_{\infty} \bigr) \cdot \max_{\vect{x}^{*} \in \mathcal{K}} F\bigl( \vect{x}^{*} \bigr)
			- O\left( \frac{\beta D^{2}}{ \ln^{2}T} \right).
$$.
\end{theorem}
\begin{proof}
Let $\vect{x}^{*} \in \mathcal{K}$ be the maximum solution of $F$.
Let $r =  e^{-\delta(1 + O(1/\ln^{2}T))} \cdot (1 - \max_{i} x^{1}_{i})$. Note that from Lemma~\ref{lem:bound-x}, it follows that $(1 - \|\vect{x}^{t}\|_{\infty}) \geq r$ for every $t$. Next we present a recursive formula in terms of $F(\vect{x}^t)$ and $F(\vect{x})$:
\begin{align*}
&2F\bigl(\vect{x}^{t+1}\bigr) - r F\bigl( \vect{x}^{*} \bigr) \\
&= 2F\bigl(( 1- \eta_{t+1}) \vect{x}^{t} + \eta_{t+1} \vect{v}^{t+1}\bigr) - r F\bigl( \vect{x}^{*} \bigr) & \tag{Using the Update step from Algorithm~\ref{algo:offline}}  \\
\\
&\geq 2F\bigl(\vect{x}^{t}\bigr) - r F\bigl( \vect{x}^{*} \bigr)
		+ 2\eta_{t+1} \langle \nabla F\bigl((1 - \eta_{t+1}) \vect{x}^{t} + \eta_{t+1} \vect{v}^{t+1}\bigr), (\vect{v}^{t+1} - \vect{x}^t) \rangle \\& \qquad  \qquad - \beta (\eta_{t+1})^{2} \| \vect{v}^{t+1} - \vect{x}^t \|^{2} \tag{$\beta$-smoothness as defined in Inequality (\ref{ineq:smooth1})} \\
\\
&= 2F\bigl(\vect{x}^{t}\bigr) - r F\bigl( \vect{x}^{*} \bigr)
		+ 2\eta_{t+1} \langle \nabla F\bigl(\vect{x}^{t}\bigr), (\vect{v}^{t+1} - \vect{x}^t) \rangle & & \\
		& \qquad \qquad + 2\eta_{t+1} \langle \nabla F\bigl((1 - \eta_{t+1}) \vect{x}^{t} + \eta_{t+1} \vect{v}^{t+1}\bigr)
				-  \nabla F\bigl(\vect{x}^{t}\bigr), (\vect{v}^{t+1} - \vect{x}^t) \rangle \\
		 & \qquad  \qquad - \beta (\eta_{t+1})^{2} \| \vect{v}^{t+1} - \vect{x}^t\|^{2}
		\\
\\
&\geq 2F\bigl(\vect{x}^{t}\bigr) - r F\bigl( \vect{x}^{*} \bigr)
		+ 2\eta_{t+1} \langle \nabla F\bigl(\vect{x}^{t}\bigr), (\vect{v}^{t+1} - \vect{x}^t) \rangle & & \\
		& \qquad  \qquad - 2\eta_{t+1} \norm{\nabla F\bigl((1 - \eta_{t+1}) \vect{x}^{t} + \eta_{t+1} \vect{v}^{t+1}\bigr)
				-  \nabla F\bigl(\vect{x}^{t}\bigr)} \norm{(\vect{v}^{t+1} - \vect{x}^t)} \\
		 &\qquad  \qquad - \beta (\eta_{t+1})^{2} \| \vect{v}^{t+1} - \vect{x}^t\|^{2} \tag{Cauchy-Schwarz} \\
\\
&\geq 2F\bigl(\vect{x}^{t}\bigr) - r F\bigl( \vect{x}^{*} \bigr)
		+  2\eta_{t+1} \langle \nabla F\bigl(\vect{x}^{t}\bigr), (\vect{v}^{t+1} - \vect{x}^t) \rangle - 3\beta (\eta_{t+1})^{2} \| \vect{v}^{t+1} - \vect{x}^t \|^{2}
		\tag{$\beta$-smoothness as defined in Inequality (\ref{ineq:smooth2})}  \\
\\
&\geq 2F\bigl(\vect{x}^{t}\bigr) - r F\bigl( \vect{x}^{*} \bigr)
		+  2\eta_{t+1} \langle \nabla F\bigl(\vect{x}^{t}\bigr), \vect{x}^{*} - \vect{x}^{t} \rangle - 3\beta (\eta_{t+1})^{2} \| \vect{v}^{t+1} - \vect{x}^t \|^{2}
		 \tag{definition of $\vect{v}^{t+1}$}  \\
\\
&\geq 2F\bigl(\vect{x}^{t}\bigr) - r F\bigl( \vect{x}^{*} \bigr)
		+  2\eta_{t+1} \left( F\bigl(\vect{x}^{*} \vee \vect{x}^{t}\bigr) - 2F\bigl(\vect{x}^{t}\bigr) \right) - 3\beta (\eta_{t+1})^{2} \| \vect{v}^{t+1} - \vect{x}^t \|^{2}
		 \tag{Lemma \ref{lem:prop2}}  \\
\\
&\geq 2F\bigl(\vect{x}^{t}\bigr) - r F\bigl( \vect{x}^{*} \bigr)
		+  2\eta_{t+1} \left( r F\bigl(\vect{x}^{*}\bigr) - 2 F\bigl(\vect{x}^{t}\bigr) \right) - 3\beta (\eta_{t+1})^{2} \| \vect{v}^{t+1} - \vect{x}^t \|^{2}
		\tag{Lemma \ref{lem:obs2}} \\
\\
%%
%&\geq 2F\bigl(\vect{x}^{t}\bigr) - r F\bigl( \vect{x}^{*} \bigr)
%		+  2\eta_{t} \left( r F\bigl(\vect{x}^{*}\bigr) - 2 F\bigl(\vect{x}^{t}\bigr) \right) - 3\beta (\eta_{t})^{2} \| \vect{v}^{t+1}\|^{2}
%		& \text{definition of $r$} & \\
%
&\geq \left( 1 - 2\eta_{t+1}\right) \left( 2F\bigl(\vect{x}^{t}\bigr) - r F\bigl( \vect{x}^{*} \bigr)  \right) - 3\beta (\eta_{t})^{2} D^{2},
\end{align*}
where $D$ is the diameter of $\mathcal{K}$.

Let $h^{t} = 2F\bigl(\vect{x}^{t}\bigr) - r F\bigl( \vect{x}^{*} \bigr)$. By the previous inequality and the choice of $\eta_{t}$, we have
$$ h^{t+1} \geq \left( 1 - 2\eta_{t+1}\right) h^{t} - 3\beta (\eta_{t})^{2} D^{2}
	= \left( 1 - 2\eta_{t+1} \right) h^{t} - O\biggl( \frac{\beta \delta^{2}D^{2}}{t^{2} \ln^{2}T} \biggr) .
$$
where we used the facts that $H_{T} = O(\ln T)$. Therefore,
\begin{align*}
h^{T} &\geq  \prod_{t=2}^{T} \left( 1 - 2\eta_{t} \right) h^{1} - O\biggl( \frac{\beta \delta^{2}D^{2}}{\ln^{2}T} \biggr) \sum_{t=1}^{T} \frac{1}{t^{2}} & & \\
&\geq e^{-2\sum_{t=2}^{T} \eta_{t} - 4\sum_{t=2}^{T} \eta_{t}^{2} } \cdot h^{1} - O\biggl( \frac{\beta \delta^{2}D^{2}}{t^{2} \ln^{2}T} \biggr) & \text{ since } (1 - u) \geq e^{-u - u^{2}} \text{ for } 0 \leq u < 1/2& \\
&= e^{-2\delta(1 + O(1/\ln^{2}T))} \cdot h^{1} - O\left(\frac{\beta \delta^{2} D^{2}}{ \ln^{2}T} \right)
& \text{ since } \sum_{t=2}^{T} \eta_{t}^{2} = \sum_{t=2}^{T} \frac{\delta^{2}}{t^{2} H_{T}^{2}} = O(1/\ln^{2}T).
\end{align*}

Hence,
\begin{align*}
2F\bigl(\vect{x}^{T}\bigr) - r F\bigl( \vect{x}^{*} \bigr)
 &\geq e^{-2\delta(1 + O(1/\ln^{2}T))} \left( 2F\bigl(\vect{x}^{1}\bigr) - r F\bigl( \vect{x}^{*} \bigr) \right)  - O \left(\frac{\beta \delta^{2} D^{2}}{ \ln^{2}T} \right)
\end{align*}
which implies,
\begin{align*}
F\bigl(\vect{x}^{T}\bigr) &\geq \frac{r}{2} \left( 1 - e^{-2\delta(1 + O(1/\ln^{2}T))} \right) F\bigl( \vect{x}^{*} \bigr) - \frac{6\beta\delta^{2} D^{2}}{ \ln^{2}T} \\
&= \frac{e^{-\delta(1 + O(1/\ln^{2}T))} \cdot (1 - e^{-2\delta(1 + O(1/\ln^{2}T))})}{2} (1 - \max_{i} x^{1}_{i}) F\bigl( \vect{x}^{*} \bigr) - O\left( \frac{\beta \delta^{2} D^{2}}{ \ln^{2}T} \right).
\end{align*}
Note that for $T$ sufficiently large, $O(1/\ln^{2}T) \ll 1$. By the choice $\delta = \left(\frac{\ln 3}{2}\right)$, we get
$$
F\bigl(\vect{x}^{T}\bigr) \geq \left(\frac{1}{3\sqrt{3}}\right)  (1 - \max_{i} x^{1}_{i}) F\bigl( \vect{x}^{*} \bigr) - O\left( \frac{\beta D^{2}}{ \ln^{2}T} \right)
$$
and the theorem follows.

\end{proof}

\begin{corollary}
If $\vect{0} \in {\cal K}$, then the guarantee in Theorem~\ref{offline-thm} can be written as: 
\begin{align*}
	F\bigl(\vect{x}^{T}\bigr) \geq \left(\frac{1}{3\sqrt{3}} \right) \max_{\vect{x}^{*} \in \mathcal{K}} F\bigl( \vect{x}^{*} \bigr) - O\left( \frac{\beta n}{ \ln^{2}T} \right).
\end{align*}
where the starting point  $\vect{x}^{1} = \vect{0}$ and the diameter $D \leq \sqrt{n}$.
\end{corollary} 
Note that inclusion of $\vect{0}$  in ${\cal K}$ does not necessarily implies that ${\cal K}$ is a down-closed polytope. 
%%%%%%%%%%%%
%%%%%%%%%%%%
%%%%%%%%%%%%
%%%%%%%%%%%%
\begin{comment}
\begin{lemma}
At every time $t$, it holds that $\vect{x}^{t} \in \mathcal{K}$ and $\vect{v}^{t} \leq \vect{1} - \vect{x}^{t-1}$.
\end{lemma}
%
\begin{proof}
First we show by induction that at every time $t$, $\vect{x}^{t} \in \mathcal{K}$. The base case where $t = 1$ is trivial.
Assume that $\vect{x}^{t-1} \in \mathcal{K}$. Then, $\vect{0} \in \mathcal{K}^{t}$. As $\mathcal{K}^{t}$ is convex
and $\vect{v}^{t} \in \mathcal{K}^{t}$, we have $\eta_{t} \vect{v}^{t} = (1-\eta_{t}) \vect{0} + \eta_{t} \vect{v}^{t}\in \mathcal{K}^{t}$.
Hence, $\vect{x}^{t} = \vect{x}^{t-1} + \eta_{t} \vect{v}^{t} \in \mathcal{K}$ because of the definition
of $\mathcal{K}^{t}$.

The inequality $\vect{v}^{t} \leq \vect{1} - \vect{x}^{t-1}$ follows immediately the fact that $\mathcal{K} \subset [0,1]^{n}$
(so $\vect{a} \leq \vect{1}$ for all $\vect{a} \in \mathcal{K}$).
\end{proof}
\end{comment}

\section{Online Continuous DR-Submodular Maximization}
\label{online-dr-submod}

We consider the DR-submodular maximization problem over  \emph{down-closed} convex sets
in the online setting. It has been observed that Stochastic Gradient Ascent performs well in practice  for
DR-submodular maximization (e.g., see Section~\ref{sec:exp}). In this section,
we establish a provable guarantee of the Gradient Ascent method by exploring useful properties
of DR-submodularity. The result can be seen as a theoretical evidence of the performance of the method.  
Note that the algorithm requires only the stochastic gradient.
Below, we present the pseudocode of our variant of the Stochastic Gradient Ascent algorithm.

\begin{algorithm}[ht]
\begin{algorithmic}[1]
\STATE $\vect{x}^{1}$ is some arbitrary point in the convex set $\mathcal{K}$.
\FOR{$t = 1$ to $T$}
		\STATE Play $\vect{x}^{t}$ and receive reward $F^t(\vect{x}^{t})$.
		\STATE Sample $\vect{g}^{t}$ such that $\E\bigl[ \vect{g}^{t} | \vect{x}^{t}\bigr] = \nabla F^t(\vect{x}^{t})$.
		\STATE Update $\vect{x}^{t+1} = \textrm{Proj}_{\mathcal{K}} \left( \vect{x}^{t} + \eta_t \vect{g}^{t} \right)$
			where $\eta_t = \left(\frac{D}{G\sqrt{t}}\right)$.
\ENDFOR
\end{algorithmic}
\caption{Online Stochastic Gradient Ascent ($\mathcal{K}$, $\eta_{t}$)}
\label{OGA}
\end{algorithm}

\begin{theorem} \label{thm:online}
Let $\mathcal{K} \subset [0,1]^{n}$ be a down-closed convex body
and assume that $F^t: \mathcal{K} \rightarrow \mathbb{R}$ are DR-submodular functions for $t = 1, 2, 3, \ldots, T$.
Let $D$ be the diameter of the convex set $\mathcal{K}$ and
and $G = \sup_{1 \leq t \leq T} \norm{\vect{g}^{t}}$.
Then for $\eta_t = \frac{D}{G\sqrt{t}}$, we have
$$
\frac{1}{4} \cdot \frac{1}{T} \sum \limits_{t = 1}^{T} F^t(\vect{x}^*) -  \frac{1}{T}  \sum \limits_{t = 1}^{T} \E\bigl[ F^t(\vect{x}^{t})\bigr] \leq  O\biggl( \frac{DG}{\sqrt{T}}\biggr).
$$
\end{theorem}
\begin{proof}
Let $\vect{z}$ be some arbitrary point in $\mathcal{K}$. Then for every $t$ we have that
\begin{align*}
	\norm{ \vect{x}^{t+1} - \vect{z} }^2 &= \norm{\textrm{Proj}_{\mathcal{K}} (\vect{x}^{t} + \eta^t \vect{g}^t) - \vect{z} }^2  &\tag{Using the Update step in Algorithm~\ref{OGA}}\\
	&\leq \norm{ \vect{x}^{t} + \eta_t \vect{g}^t  - \vect{z} }^2  \tag{by Pythagorean inequality} \\
	&=  \norm{  \vect{x}^{t} - \vect{z}}^2 + \eta^2_t \norm{\vect{g}^t }^2 - 2 \eta_t \langle \vect{g}^t, \vect{z} - \vect{x}^{t}\rangle.
\end{align*}
Rearranging the last inequality and note that $\norm{\vect{g}^t }^2 \leq G^{2}$, we have
\begin{align}	\label{eq:grad-nabla}
\langle \vect{g}^t, \vect{z} - \vect{x}^{t}\rangle \leq \frac{ \norm{  \vect{x}^{t} - \vect{z}}^2 - \norm{ \vect{x}^{t+1} - \vect{z} }^2 +  \eta^2_t G^2 }{2 \eta_t }
\end{align}

Define $\vect{z}^{t} := (\vect{x}^* \vee \vect{x}^{t}) - \vect{x}^{t}$ for every $t$. As $\mathcal{K}$ is down-closed,
$\vect{z}^{t} \in \mathcal{K}$.
We have
\begin{align}
& \frac{1}{4} \sum \limits_{t = 1}^T F^t(\vect{x}^*) - \sum \limits_{t = 1}^T \E\bigl[ F^t(\vect{x}^{t})\bigr] \notag \\
&\leq \frac{1}{4} \E\biggl[ \sum \limits_{t = 1}^T  \left(    F^t(\vect{x}^* \vee \vect{x}^{t})   + F^t(\vect{x}^* \wedge \vect{x}^{t}) + F^t(\vect{z}^{t} \vee \vect{x}^{t}) + F^t(\vect{z}^{t} \wedge \vect{x}^{t}) \right) - 4 \sum \limits_{t = 1}^T  F^t(\vect{x}^{t}) \biggr]  \tag{Using Lemma~\ref{lem:prop1}} \\
&= \frac{1}{4} \E\biggl[ \sum \limits_{t = 1}^T  \left(    F^t(\vect{x}^* \vee \vect{x}^{t})   + F^t(\vect{x}^* \wedge \vect{x}^{t}) + F^t(\vect{z}^{t} \vee \vect{x}^{t}) + F^t(\vect{z}^{t} \wedge \vect{x}^{t})  - 4 F^t(\vect{x}^{t}) \right)  \biggr]  \notag \\
&= \frac{1}{4} \E\biggl[ \sum \limits_{t = 1}^T  \left(    F^t(\vect{x}^* \vee \vect{x}^{t})   + F^t(\vect{x}^* \wedge \vect{x}^{t}) - 2F^t(\vect{x}^{t}) + F^t(\vect{z}^{t} \vee \vect{x}^{t}) + F^t(\vect{z}^{t} \wedge \vect{x}^{t})  - 2F^t(\vect{x}^{t}) \right)  \biggr] \notag \\
&\leq \frac{1}{4} \E \biggl[ \sum \limits_{t = 1}^T  \left( \langle \vect{g}^t, \vect{x}^* - \vect{x}^{t} \rangle
	+   \langle \vect{g}^t, \vect{z}^{t} - \vect{x}^{t} \rangle   \right)  \biggr].
	 \label{eq:online-crucial}
\end{align}
The last inequality is due to Lemma~\ref{lem:prop2}; and $\E\bigl[ \vect{g}^{t} | \vect{x}^{t}\bigr] = \nabla F^t(\vect{x}^{t})$;
and linearity of expectation.

We bound the first term in \eqref{eq:online-crucial}. By Inequality (\ref{eq:grad-nabla}), we have
\begin{align}		
\sum \limits_{t = 1}^T  & \langle \vect{g}^t, \vect{x}^* - \vect{x}^{t} \rangle
\leq \sum \limits_{t = 1}^T \frac{ \norm{  \vect{x}^{t} - \vect{x}^*}^2 - \norm{ \vect{x}^{t+1} - \vect{x}^* }^2 +  \eta^2_t G^2 }{2 \eta_t }  \notag \\
&\leq  \sum \limits_{t = 2}^T \norm{ \vect{x}^{t} - \vect{x}^* }^2 \left( \frac{1}{2\eta_t} - \frac{1}{2\eta_{t-1}} \right) + \frac{1}{2\eta_1}  \norm{ \vect{x}^1 - \vect{x}^*}^2 +  \sum \limits_{t = 1}^T \frac{\eta_t}{2} G^2   \notag \\
&\leq  \norm{ \vect{x}^{t} - \vect{x}^*}^2 \left(\frac{G}{D}\right)  \sum \limits_{t = 2}^T \left( \frac{\sqrt{t}}{2} - \frac{\sqrt{t-1}}{2} \right) + \frac{1}{2\eta_1}  \norm{ \vect{x}^1 - \vect{x}^*}^2 +  \sum \limits_{t = 2}^T  \frac{\eta_t}{2} G^2  \tag{replacing $\eta_t = \frac{D}{G \sqrt{2}}$} \\
&\leq  DG \sum \limits_{t = 2}^T  \frac{1}{4\sqrt{t-1}}
	+ \frac{DG}{2}
	+ \frac{DG}{2} \sum \limits_{t = 2}^T  \frac{1}{\sqrt{t}}  \notag \\
&= O\bigl(DG\sqrt{T}\bigr).
\label{eq:online-crucial-1}
\end{align}

We are now bounding the second term in (\ref{eq:online-crucial}). Observe that
\begin{align}	\label{eq:z}
\norm{ \vect{z}^{t} - \vect{z}^{t-1} }^{2}
&= \norm{ \bigl( \vect{x}^{*} \vee \vect{x}^{t} \bigr) -  \vect{x}^{t} - \bigl( \vect{x}^{*} \vee \vect{x}^{t-1} \bigr) + \vect{x}^{t-1}  }^{2} \notag  \\
&= \norm{ \bigl( \vect{x}^{*} \vee \vect{x}^{t} \bigr) - \bigl( \vect{x}^{*} \vee \vect{x}^{t-1} \bigr) - \bigl( \vect{x}^{t} - \vect{x}^{t-1} \bigr) }^{2} \notag \\
&\leq 2 \norm{ \vect{x}^{t} - \vect{x}^{t-1} }^{2} \leq 2\eta^{2}_{t-1} \norm{\vect{g}^{t-1}}^2
\leq 2\eta^{2}_{t-1} G^2,
\end{align}
where the last inequality follows the algorithm and the Cauchy-Schwarz inequality.
Now we have
\begin{align}	\label{eq:online-crucial-2}
&\sum \limits_{t = 1}^T  \langle \vect{g}^t, \vect{z}^{t} - \vect{x}^{t} \rangle
\leq \sum \limits_{t = 1}^T  \frac{ \norm{  \vect{x}^{t} - \vect{z}^{t}}^2 - \norm{ \vect{x}^{t+1} - \vect{z}^{t} }^2 +  \eta^2_t G^2 }{2 \eta_t } 	\notag \\
&\leq \sum \limits_{t = 1}^T  \frac{ \norm{  \vect{x}^{t} - \vect{z}^{t-1}}^2 + \norm{  \vect{z}^{t} - \vect{z}^{t-1}}^2 - \norm{ \vect{x}^{t+1} - \vect{z}^{t} }^2 +  \eta^2_t G^2 }{2 \eta_t }  \notag \\
&\leq  \sum \limits_{t = 2}^T \biggl( \frac{1}{2\eta_t} - \frac{1}{2\eta_{t-1}} \biggr)  \norm{ \vect{x}^{t} - \vect{z}^{t-1}}^2
	+ \frac{\norm{ \vect{x}^{1} - \vect{z}^{1}}^2}{2\eta_1}
	+ \sum \limits_{t = 1}^T  \frac{ \norm{  \vect{z}^{t} - \vect{z}^{t-1} }^2}{2 \eta_t }
	+  \sum \limits_{t = 2}^T \frac{\eta_t G^2}{2}    \notag \\
&\leq \sum \limits_{t = 2}^T \biggl( \frac{1}{2\eta_t} - \frac{1}{2\eta_{t-1}} \biggr) D^2
	+ \frac{1}{2\eta_1}   D^2
	+  \sum \limits_{t = 2}^T \biggl( \frac{\eta^{2}_{t-1}}{\eta_{t}} + \frac{\eta_t }{2} \biggr) G^2  \notag \\
&\leq  DG\sum \limits_{t = 2}^T \biggl( \frac{\sqrt{t}}{2} - \frac{\sqrt{t-1}}{2} \biggr)
	+ \frac{DG}{2}
	+  3\sum \limits_{t = 2}^T \eta_t  G^2  \notag \\
&\leq  DG \frac{\sqrt{T}}{2}
	+ \frac{DG}{2}
	+  3DG \sum \limits_{t = 2}^T \frac{1}{\sqrt{t}} \notag \\
&= O\bigl(DG\sqrt{T}\bigr).
\end{align}
The first inequality follows from Inequality (\ref{eq:grad-nabla}). The fourth inequality is due to Inequality (\ref{eq:z}). The fifth inequality holds since $\eta^{2}_{t-1} \leq 2 \eta^{2}_{t}$. The theorem follows from the Inequalities (\ref{eq:online-crucial}), (\ref{eq:online-crucial-1}) and (\ref{eq:online-crucial-2}).
\end{proof}

\section{Experiments} 	\label{sec:exp}

In this section, we validate offline and online algorithms for non-monotone DR submodular optimization on both, the real-world and the synthetic datasets. Our experiments are broadly classified into following three categories:
\begin{enumerate}
\item We compare our \emph{offline} algorithm (from Section~\ref{offline-dr-submod}) against two previous known algorithms for maximizing non-monotone DR-submodular function over \emph{down-closed} polytopes mentioned in~\cite{BianLevy17:Non-monotone-Continuous}. Recall that our algorithm applies to a more general setting where the convex set is not required to be down-closed.
\item Next, we show the performance of our \emph{offline} algorithm  (from Section~\ref{offline-dr-submod}) for maximizing non-monotone DR-submodular function over \emph{general} polytopes. Recall that no previous algorithm was known to have performance guarantees for this problem.
\item Finally, we show the performance of our \emph{online} algorithm (from Section~\ref{online-dr-submod}) for maximizing non-montone DR-submodular function over \emph{down-closed} polytopes.
\end{enumerate}
All experiments  are performed in MATLAB using CPLEX optimization tool on MAC OS version 10.14. 

\subsection{Offline Algorithm over Down-closed Polytopes}

Here, we benchmark the performance of our variant of the Frank-Wolfe algorithm from Section~\ref{offline-dr-submod} against the previous known two algorithms for maximizing continuous DR submodular function over \emph{down-closed} polytopes mentioned in~\cite{BianLevy17:Non-monotone-Continuous}. We considered QUADPROGIP, which is a global solver for non-convex quadratic programming, as a baseline.  We run all the algorithms for 100 iterations. All the results are the average of 20 repeated experiments. 
%and the original source code can be found at: https://github.com/bianan/non-monotone-dr-submodular \thang{remove it}.
For the sake of completion, we describe below the problem and different settings used.
%These are similar experimental settings as mentioned in~\cite{BianLevy17:Non-monotone-Continuous}.
We follow closely the experimental settings from~\cite{BianLevy17:Non-monotone-Continuous}, and adapted their source codes to our algorithms.

\subsubsection{Quadratic Programming}

As a state-of-the-art global solver, we used QUADPROGIP to find the global optimum which is  used to calculate the approximation ratios. Our problem instances are synthetic quadratic objectives with down-closed polytope constraints, i.e.,
$$f(x) = \frac{1}{2} \vect{x}^{\top} \vect{H} \vect{x} + \vect{h}^{\top} \vect{x} + c$$
and
$$\mathcal{K} = \{ \vect{x} \in \mathbb{R}^n_{+} | \vect{A} \vect{x} \leq \vect{b}, \vect{x} \leq \vect{u}, \vect{A} \in \mathbb{R}_+^{m \times n}, \vect{b} \in \mathbb{R}_+^{m} \}.$$
Note that in previous sections, we have assumed w.l.o.g that $\mathcal{K} \subseteq [0,1]^{n}$.   By scaling our results hold as well for the general box constraint $\vect{x} \leq \vect{u}$, provided the entries of $u$ are upper bounded by a constant.
% In this section, for simplicity (in order to avoid normalizing at every step), we consider the box constraint
% $\vect{x} \leq \vect{u}$ instead
% of the constraint $\vect{x} \leq \vect{1}$.

%
Both objective and constraints were randomly generated, using the following two ways:

\medskip
\noindent \textbf{Uniform Distribution}:  $\vect{H} \in \mathbb{R}^{n \times n}$
is a symmetric matrix with uniformly distributed entries in $[-1, 0]$; $\vect{A} \in \mathbb{R}^{m \times n}$ has uniformly distributed entries in $[\mu, \mu + 1]$, where $\mu$ = 0.01 is a small positive constant in order to make entries of $\vect{A}$ strictly positive for down-closed polytope.
%In case of general polytope, we set all the entries in the last row of  $\vect{A}$ to $-1$.

\begin{figure}[h]
\centerline {
	\parbox{0.33\textwidth}{\includegraphics [width=0.33\textwidth] {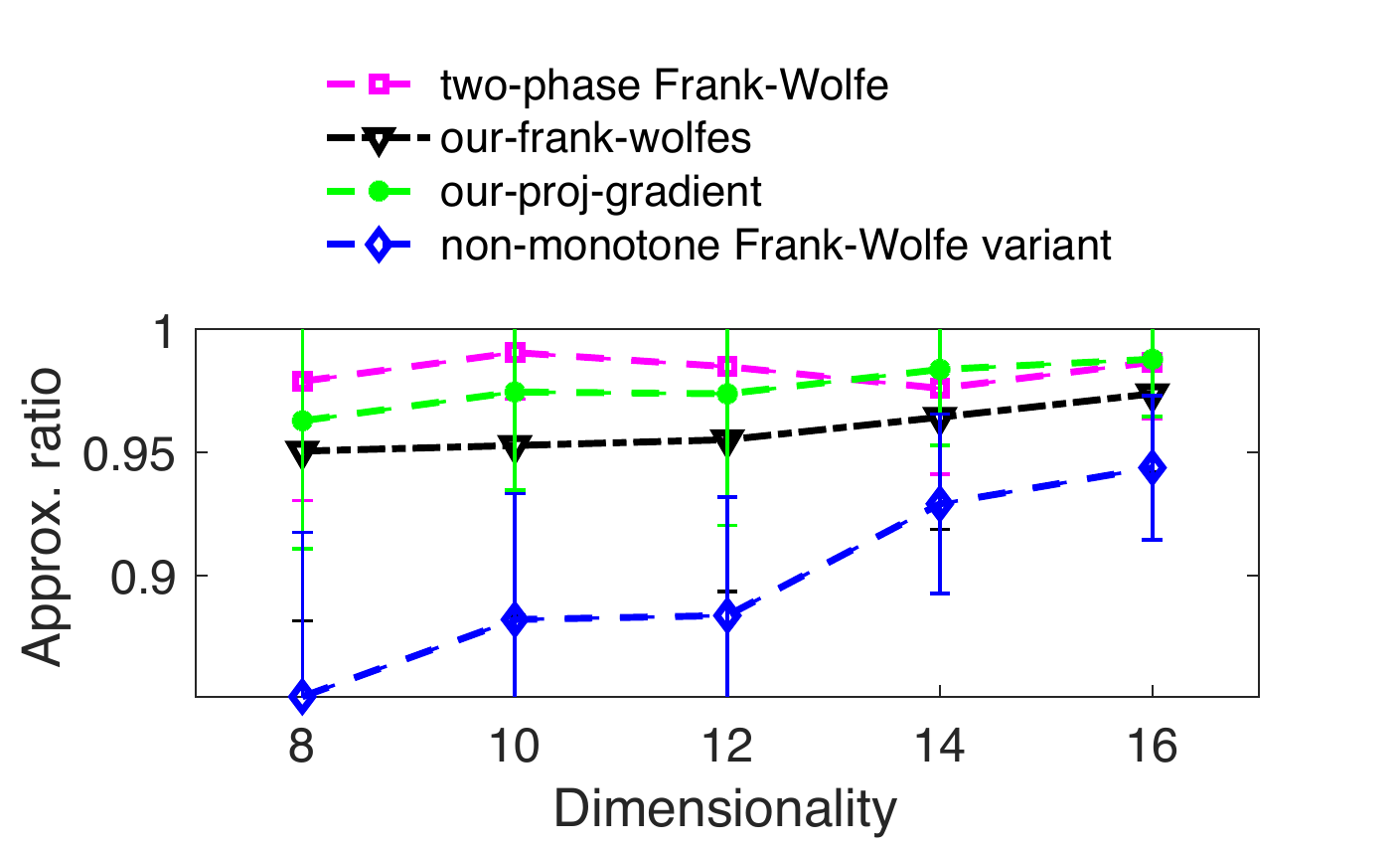}}
	\parbox{0.33\textwidth}{\includegraphics [width=0.33\textwidth] {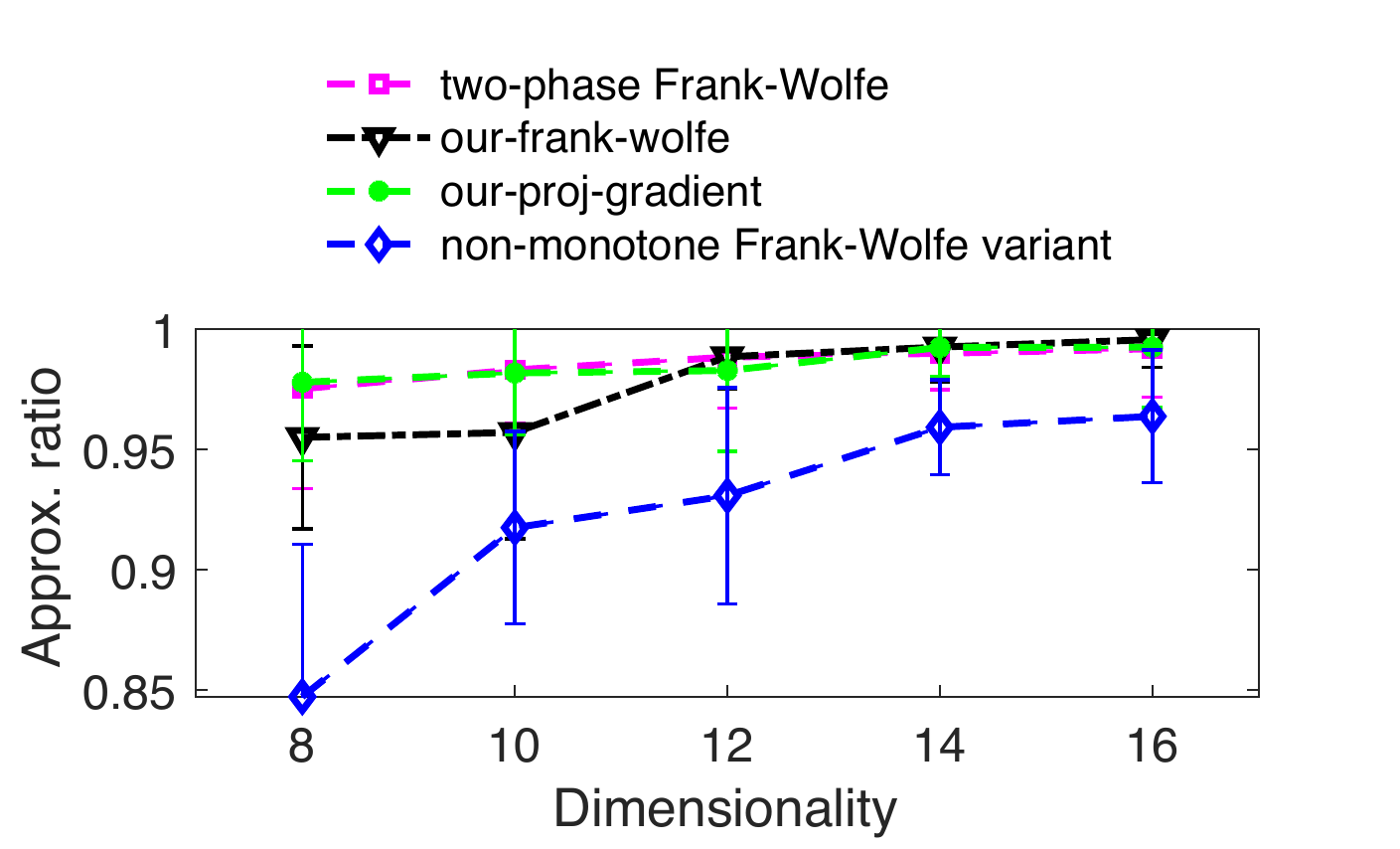}}
	\parbox{0.33\textwidth}{\includegraphics [width=0.33\textwidth] {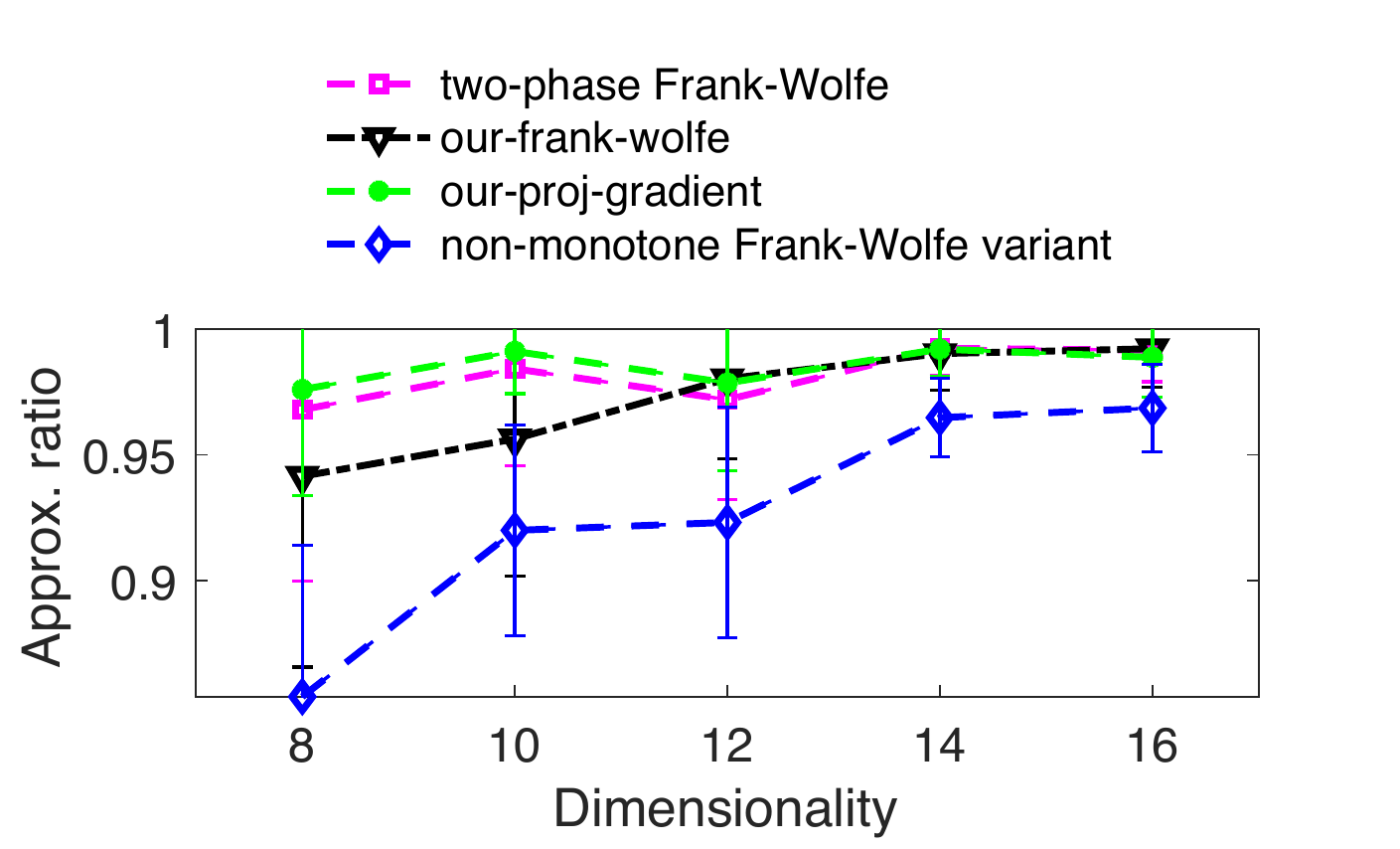}}
	}
\centerline
{
\parbox{0.33\textwidth}{\centering{(a)}}
\parbox{0.33\textwidth}{\centering{(b)}}
\parbox{0.33\textwidth}{\centering{(c)}}
}
\caption{Uniform distribution with down-closed polytope and (a) $m = \lfloor 0.5 n \rfloor$ (b)  $m = n$ (c) $m = \lfloor 1.5 n \rfloor$.}
\label{Down-Uni}
\end{figure}
%
%\begin{figure}[h]
%\centerline {
%	\parbox{0.33\textwidth}{\includegraphics [width=0.33\textwidth] {Results/Gen-poly/uni/Gen_quad_uniform_m-half-n-n_exp20-seed0.pdf}}
%	\parbox{0.33\textwidth}{\includegraphics [width=0.33\textwidth] {Results/Gen-poly/uni/Gen_quad_uniform_m-n-n_exp20-seed0.pdf}}
%	\parbox{0.33\textwidth}{\includegraphics [width=0.33\textwidth] {Results/Gen-poly/uni/Gen_quad_uniform_m-onehalf-n-n_exp20-seed0.pdf}}
%	}
%\centerline
%{
%\parbox{0.33\textwidth}{\centering{(a)}}
%\parbox{0.33\textwidth}{\centering{(b)}}
%\parbox{0.33\textwidth}{\centering{(c)}}
%}
%\caption{Uniform distribution with general polytope and (a) $m = \lfloor 0.5 n \rfloor$ (b)  $m = n$ (c) $m = \lfloor 1.5n \rfloor$ }
%\label{Gen-Uni}
%\end{figure}

%
\medskip
\noindent \textbf{Exponential Distribution}: Here, the entries of $-\vect{H}$ and $\vect{A}$ are sampled from exponential distributions $Exp(\lambda)$ where given a random variable $y \geq 0$, the probability density function of $y$ is defined by $\lambda e^{-\lambda y}$, and for $y < 0$, its density is fixed to be $0$. Specifically, each entry of H is sampled from $Exp(1)$, then the matrix $-\vect{H}$ is made to be symmetric. Each entry of $\vect{A}$ is sampled from $Exp(0.25) + \mu$, where $\mu = 0.01$ is a small positive constant.
%In case of general polytope, we set all the entries in the last row of  $\vect{A}$ to $-1$.
%

\medskip
%In both the above two cases of down-closed polytope,
We set $\vect{b} = \vect{1}^m$, and $\vect{u}$ to be the tightest upper bound of $\mathcal{K}$ by $u_j = \min_{i \in [m]} \frac{b_i}{A_{ij}}, \forall j \in [n]$.
%In case of general polytope and uniform distribution, we set the $b_m = -1$ whereas we set $b_m = -0.1$ for the exponential distribution.
In order to make $f$ non-monotone, we set $\vect{h} = -0.2 * \vect{H}^{\top} \vect{u}$. To make sure that $f$ is non-negative, we first of all solve the problem $\min_{\vect{x} \in \mathcal{P}} \frac{1}{2} \vect{x}^{\top} \vect{H} \vect{x} + \vect{h}^{\top} \vect{x}$ using QUADPROGIP. Let the solution to be $\hat{x}$, then we set $c = - f(\hat{x}) + 0.1 * |f(\hat{x})|$.

The approximation ratios w.r.t.\ dimensionalities $(n)$ are plotted in Figures~\ref{Down-Uni} and \ref{Down-Exp} for the two distributions.
%where Figures~\ref{Down-Uni} and~\ref{Down-Exp} correspond to down-closed polytopes and Figures~\ref{Gen-Uni} and~\ref{Gen-Exp} correspond to general polytopes.
In each figure, we set the number of constraints to be $m = \lfloor 0.5n \rfloor$, $m = n$ and $m = \lfloor1.5n \rfloor$, respectively.

We can observe that our version of Frank-Wolfe (denoted our-frank-wolfe) and gradient ascent algorithm (denoted by proj-gradient) have comparable performance with the state-of-the-art algorithms when optimizing submodular functions over down-closed convex sets. Note that the performance is clearly consistent with the proven approximation guarantee of $1/e$ shown in \cite{BianLevy17:Non-monotone-Continuous}.
%In the case of general convex set, we show the performance of our algorithms with respect to the solution found using \textsc{QUADPROG}IP. Recall that no previous algorithm is known for this case.
We also show that the performance of our algorithms are consistent with the proven approximation guarantee of $1/3\sqrt3$ for down-closed convex sets.

\begin{figure}[th]
\centerline {
	\parbox{0.33\textwidth}{\includegraphics [width=0.33\textwidth] {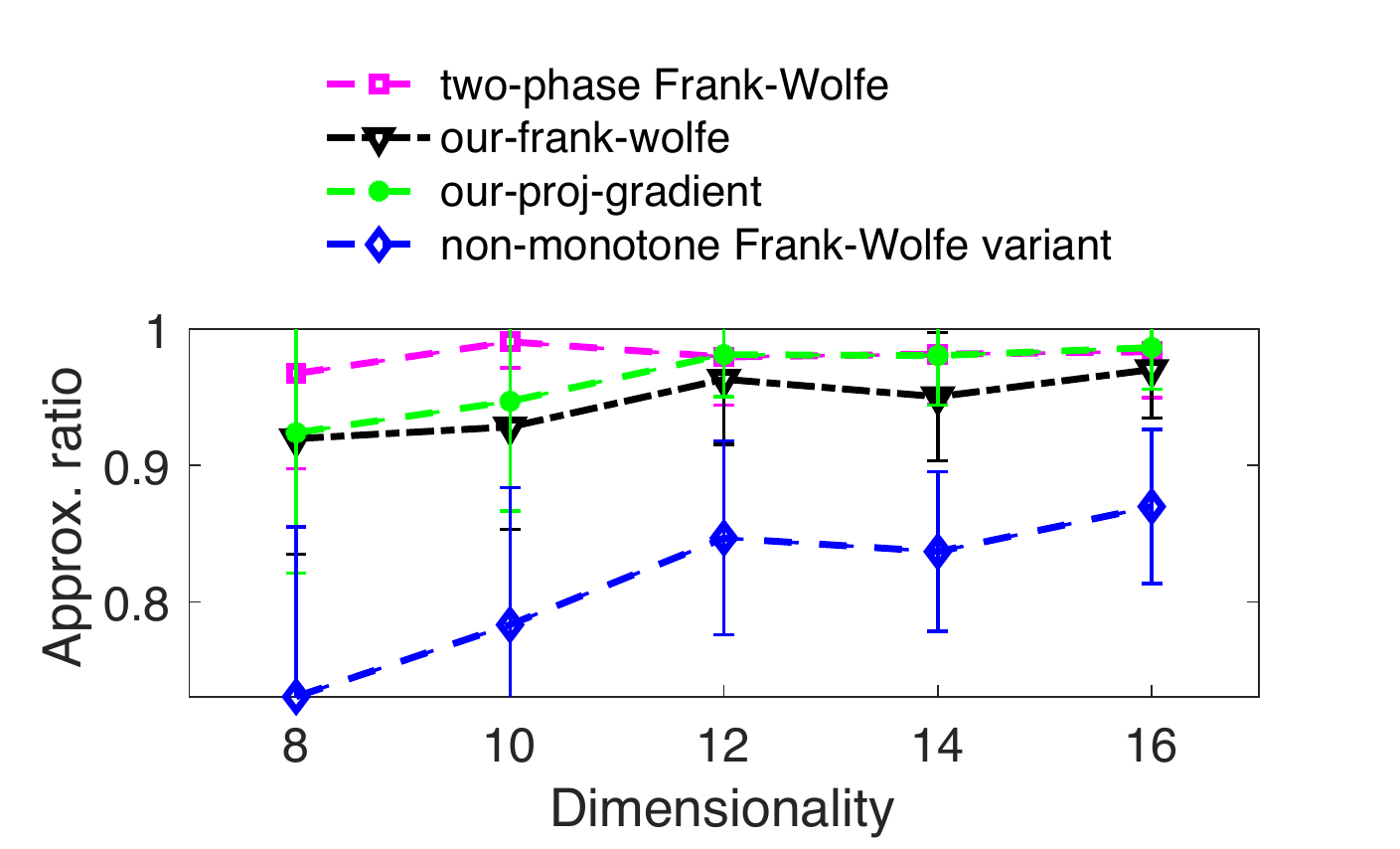}}
	\parbox{0.33\textwidth}{\includegraphics [width=0.33\textwidth] {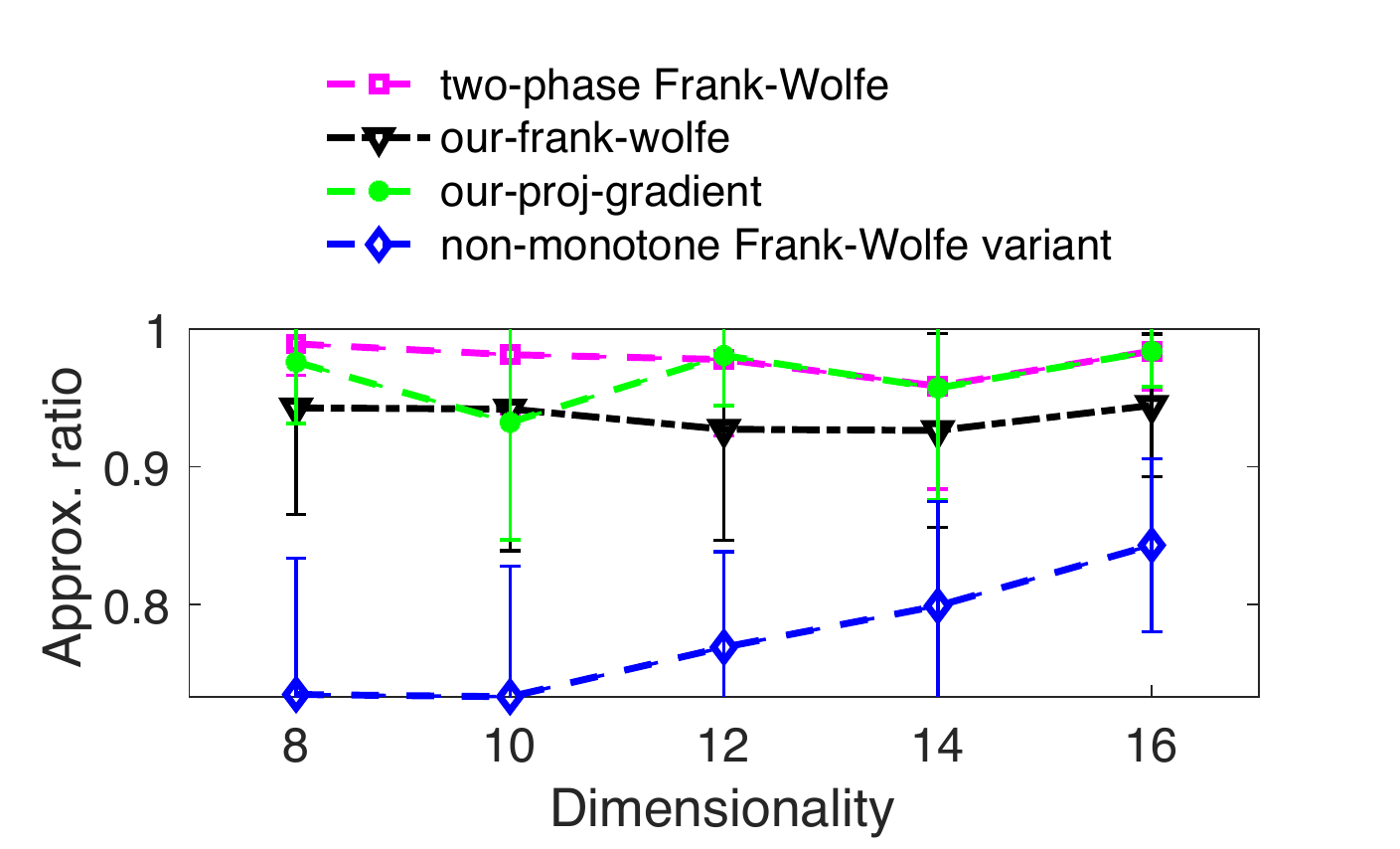}}
	\parbox{0.33\textwidth}{\includegraphics [width=0.33\textwidth] {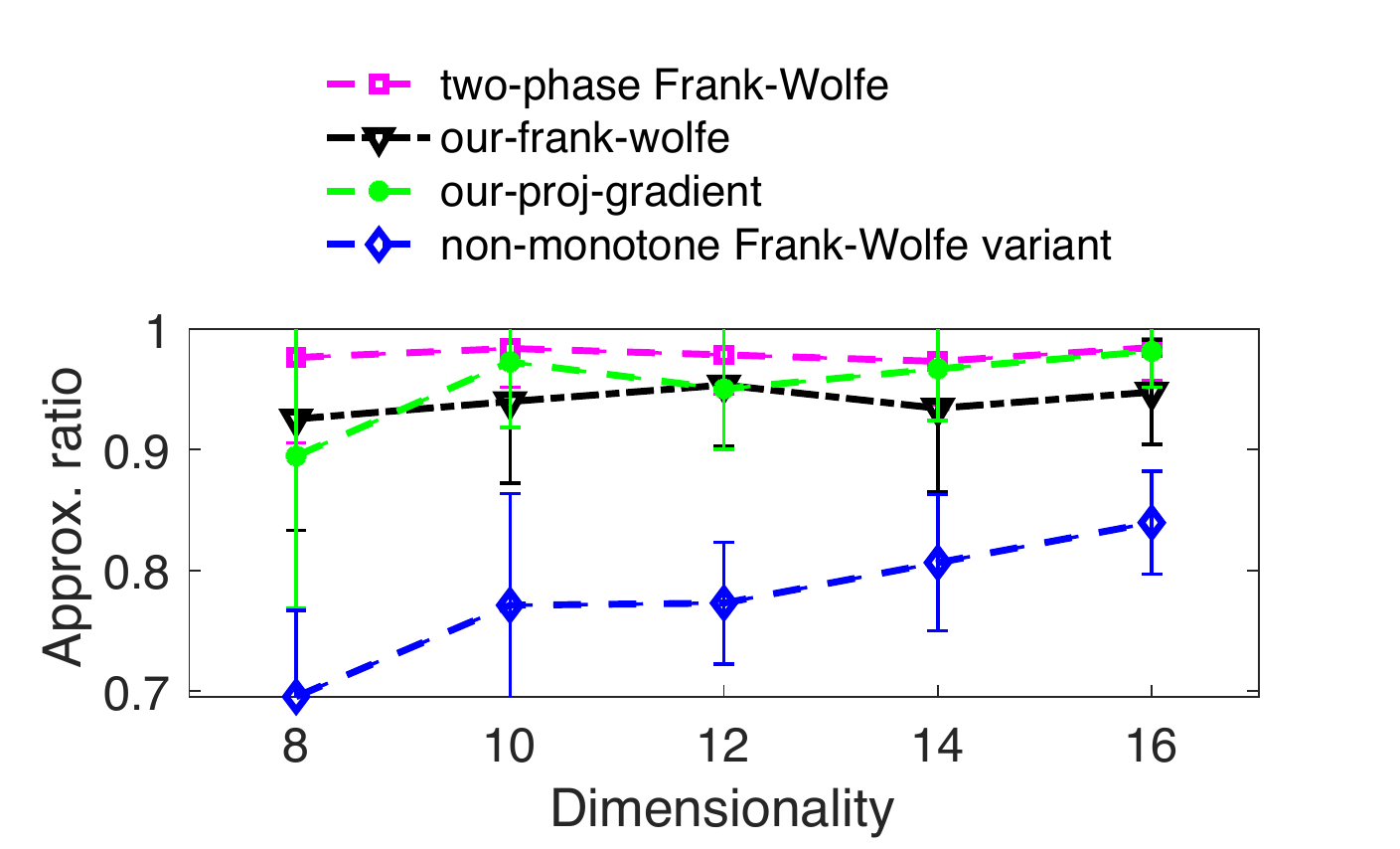}}
	}
\centerline
{
\parbox{0.33\textwidth}{\centering{(a)}}
\parbox{0.33\textwidth}{\centering{(b)}}
\parbox{0.33\textwidth}{\centering{(c)}}
}
\caption{Exponential distribution with down-closed polytope and (a) $m = \lfloor 0.5 n \rfloor$, (b)  $m = n$, (c) $m = \lfloor 1.5 n \rfloor$}
\label{Down-Exp}
\end{figure}
%
%\begin{figure}[th]
%\centerline {
%	\parbox{0.33\textwidth}{\includegraphics [width=0.33\textwidth] {Results/Gen-poly/exp/Gen_quad_exp_m-half-n-n_exp20-seed0.pdf}}
%	\parbox{0.33\textwidth}{\includegraphics [width=0.33\textwidth] {Results/Gen-poly/exp/Gen_quad_exp_m-n-n_exp20-seed0.pdf}}
%	\parbox{0.33\textwidth}{\includegraphics [width=0.33\textwidth] {Results/Gen-poly/exp/Gen_quad_exp_m-onehalf-n-n_exp20-seed0.pdf}}
%	}
%\centerline
%{
%\parbox{0.33\textwidth}{\centering{(a)}}
%\parbox{0.33\textwidth}{\centering{(b)}}
%\parbox{0.33\textwidth}{\centering{(c)}}
%}
%\caption{Exponential distribution with general polytope and (a) $m = \lfloor 0.5 n \rfloor$ (b)  $m = n$}
%\label{Gen-Exp}
%\end{figure}
%
%
%

%%%%%%*********************
%%%%%%*********************
\subsubsection{Maximizing Softmax Extentions}

Determinantal point processes (DPPs) are probabilistic models of repulsion, that have been used to model diversity in machine learning~\cite{KuleszaTaskar12:Determinantal-point}. The constrained MAP (maximum a posteriori) inference problem of a DPP is an NP-hard combinatorial problem. One of the current methods with the best known approximation guarantee is based on the softmax extension~\cite{GillenwaterKulesza12:Near-optimal-map-inference}, which is a DR-submodular function. Let $\vect{L}$ be the positive semidefinite kernel matrix of a DPP, its softmax extension is:
$$f(\vect{x}) = \log det(diag(\vect{x})(\vect{L} - \vect{I}) + \vect{I}), \vect{x} \in [0, 1]^n, $$
where $\vect{I}$ is the identity matrix, $diag(\vect{x})$ is the diagonal matrix with diagonal elements set as $\vect{x}$. The problem of MAP inference in DPPs corresponds to the problem of maximizing $f$ over a convex polytope $\mathcal{K}$.

\begin{figure}[h]
\centerline {
	\parbox{0.33\textwidth}{\includegraphics [width=0.33\textwidth] {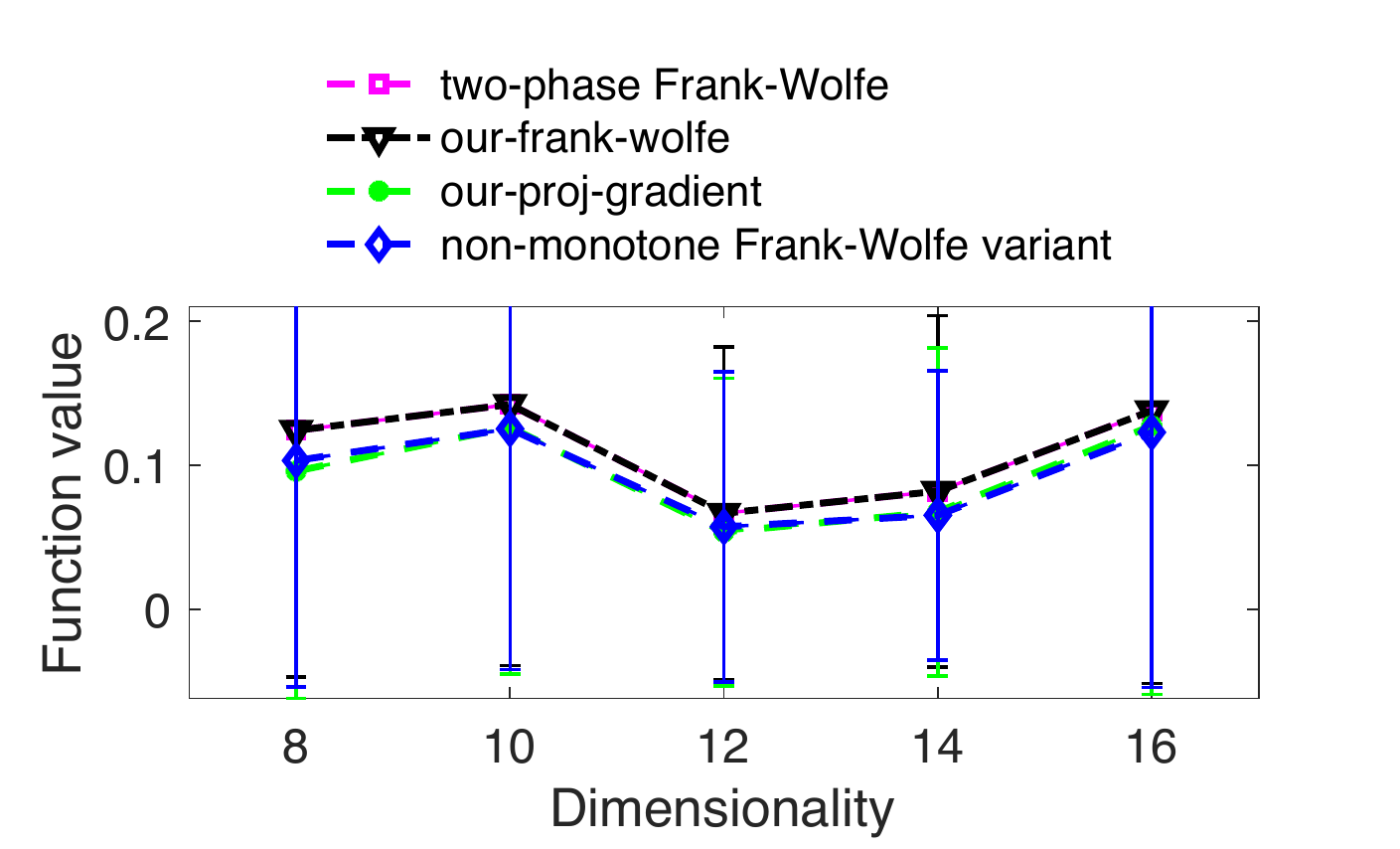}}
	\parbox{0.33\textwidth}{\includegraphics [width=0.33\textwidth] {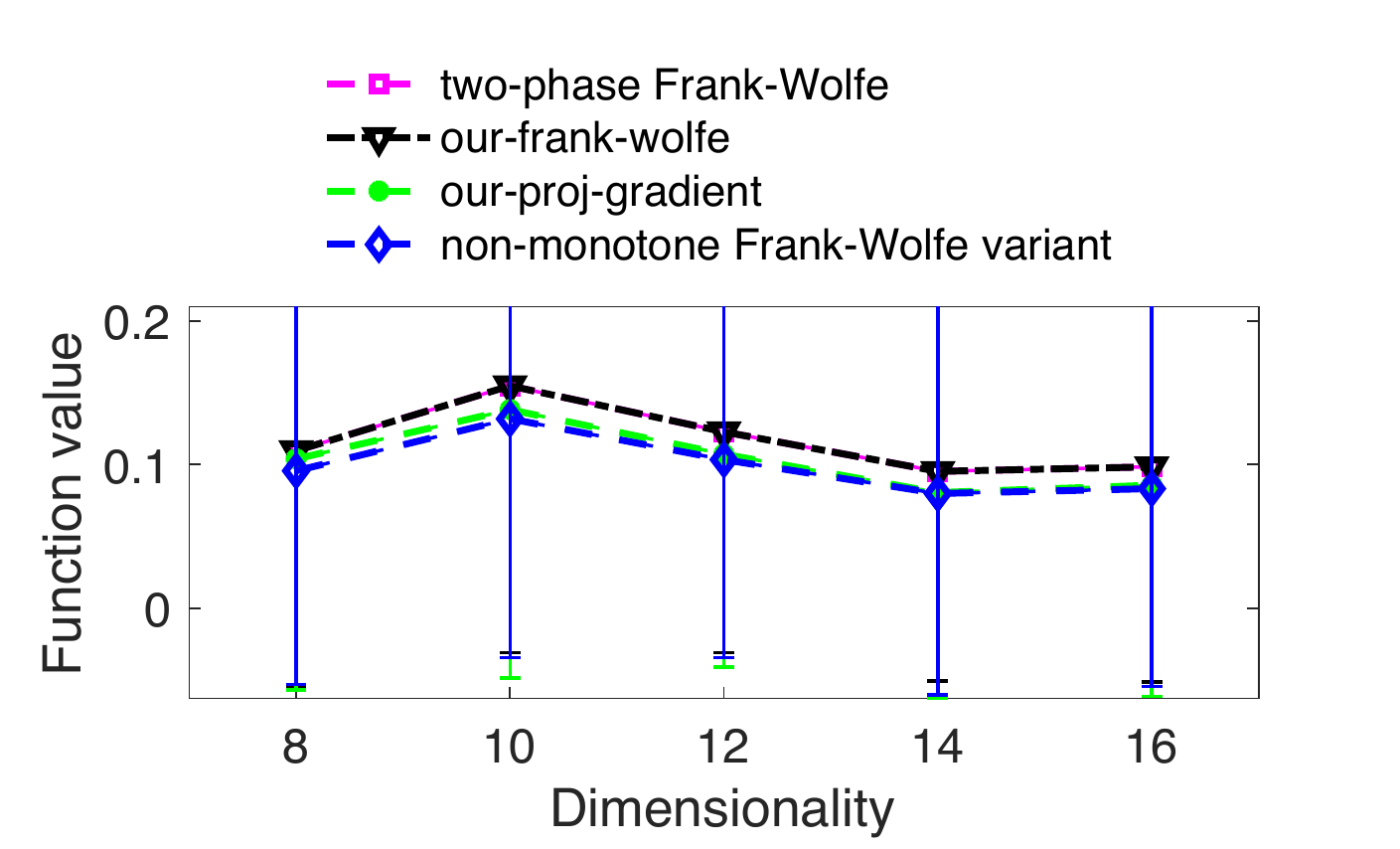}}
	\parbox{0.33\textwidth}{\includegraphics [width=0.33\textwidth] {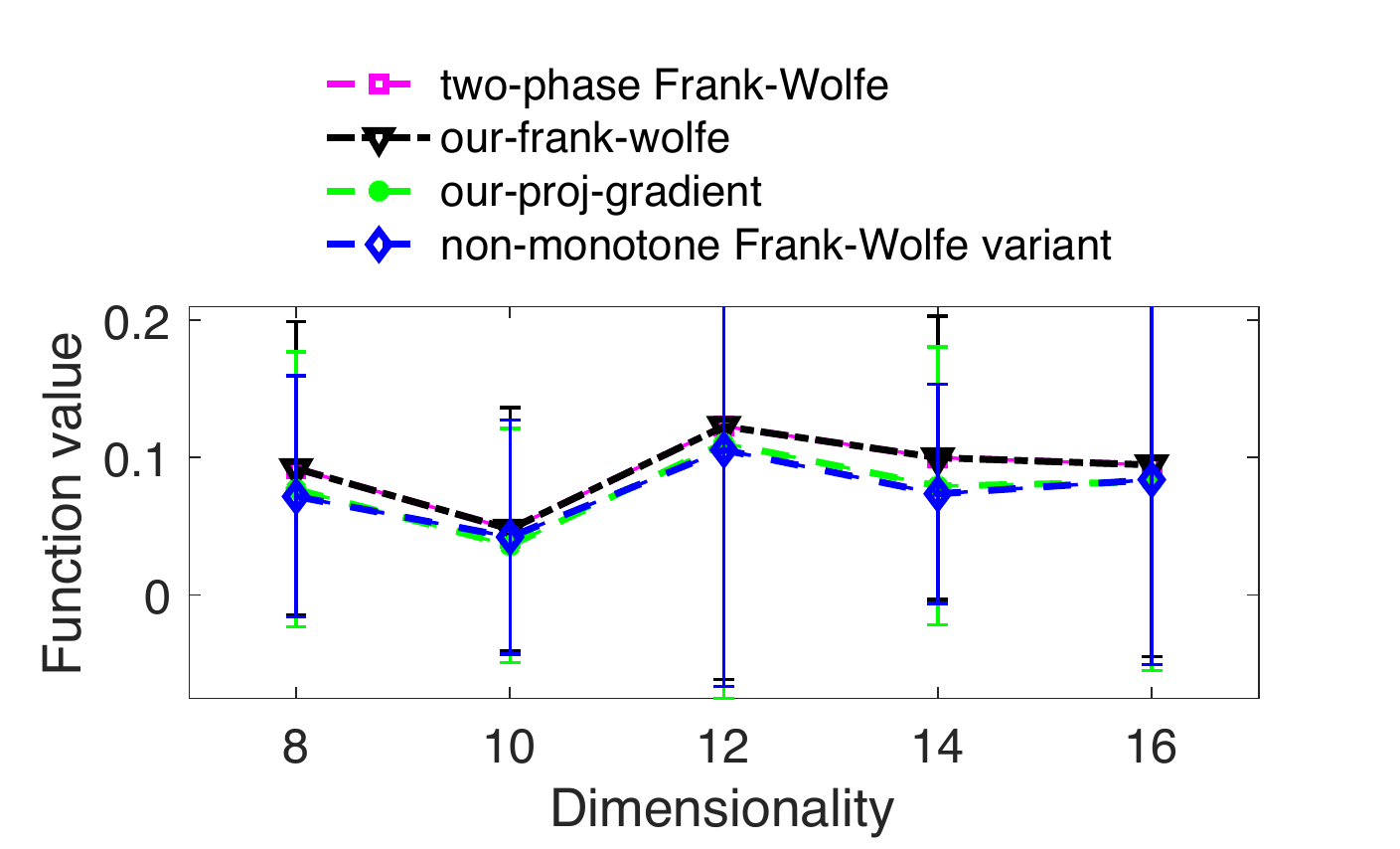}}
	}
\centerline
{
\parbox{0.33\textwidth}{\centering{(a)}}
\parbox{0.33\textwidth}{\centering{(b)}}
\parbox{0.33\textwidth}{\centering{(c)}}
}
\caption{Softmax extension in uniform distribution with down-closed polytope and (a) $m = \lfloor 0.5 n \rfloor$, (b)  $m = n$(c) $m = \lfloor 1.5 n \rfloor$}
\label{Soft-Uni}
\end{figure}
\begin{figure}[h]
\centerline {
	\parbox{0.33\textwidth}{\includegraphics [width=0.33\textwidth] {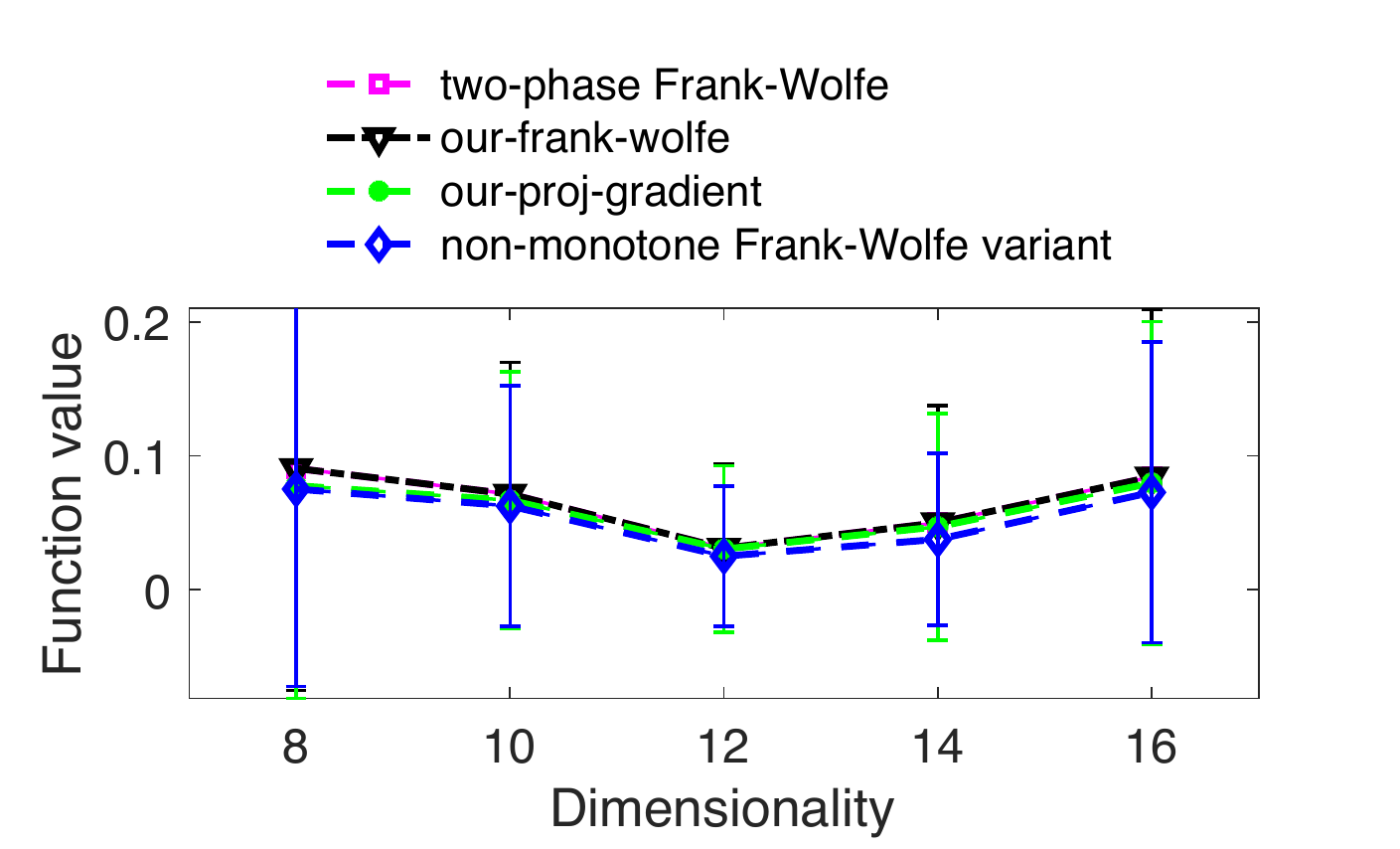}}
	\parbox{0.33\textwidth}{\includegraphics [width=0.33\textwidth] {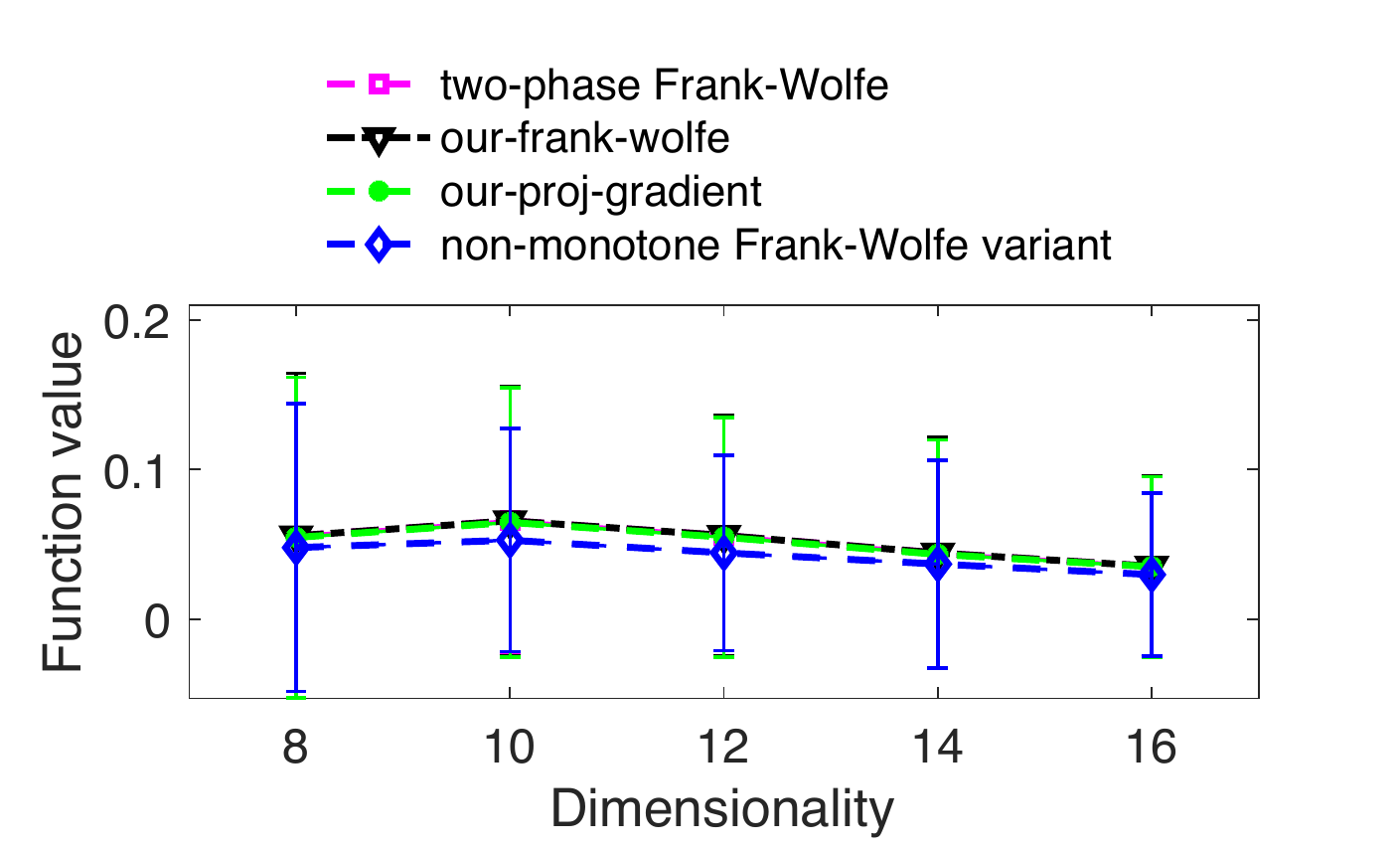}}
	\parbox{0.33\textwidth}{\includegraphics [width=0.33\textwidth] {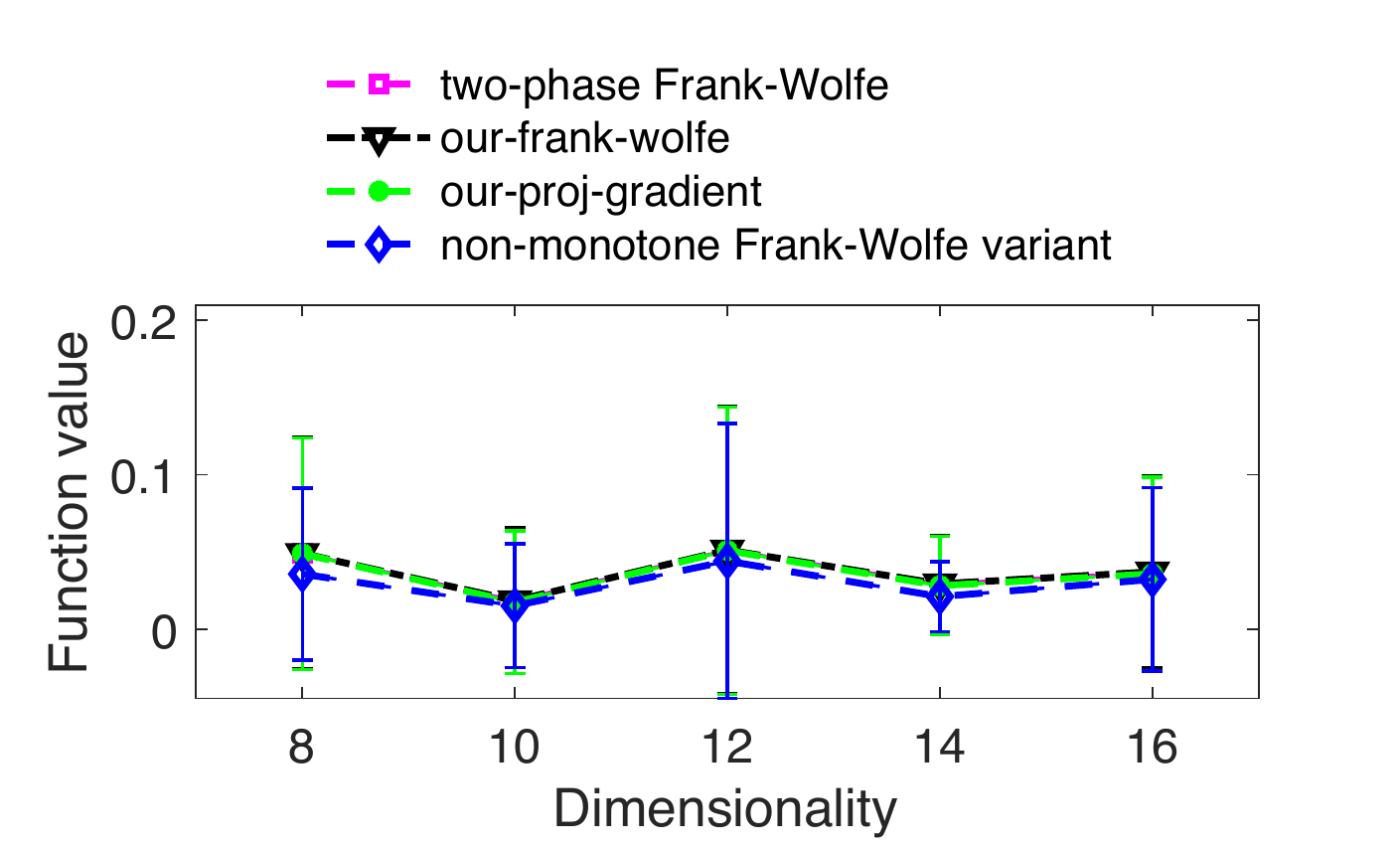}}
	}
\centerline
{
\parbox{0.33\textwidth}{\centering{(a)}}
\parbox{0.33\textwidth}{\centering{(b)}}
\parbox{0.33\textwidth}{\centering{(c)}}
}
\caption{Softmax extension in exponential distribution with down-closed polytope and (a) $m = \lfloor 0.5 n \rfloor$ (b)  $m = n$}
\label{Soft-Exp}
\end{figure}
%
%\begin{figure}[h]
%\centerline {
%	\parbox{0.33\textwidth}{\includegraphics [width=0.33\textwidth] {Results/Gen-poly/Soft-uni/gen_softmax_uniform_m-half-n-n_exp20-seed0.pdf}}
%	\parbox{0.33\textwidth}{\includegraphics [width=0.33\textwidth] {Results/Gen-poly/Soft-uni/gen_softmax_uniform_m-n-n_exp20-seed0.pdf}}
%	\parbox{0.33\textwidth}{\includegraphics [width=0.33\textwidth] {Results/Gen-poly/Soft-uni/gen_softmax_uniform_m-onehalf-n-n_exp20-seed0.pdf}}
%	}
%\centerline
%{
%\parbox{0.33\textwidth}{\centering{(a)}}
%\parbox{0.33\textwidth}{\centering{(b)}}
%\parbox{0.33\textwidth}{\centering{(c)}}
%}
%\caption{ Softmax extension in uniform distribution with general polytope and (a) $m = \lfloor 0.5 n \rfloor$ (b)  $m = n$}
%\label{Gen-Soft-Uni}
%\end{figure}
%
%\begin{figure}[h]
%\centerline {
%	\parbox{0.33\textwidth}{\includegraphics [width=0.33\textwidth] {Results/Gen-poly/Soft-exp/gen_softmax_exp_m-half-n-n_exp20-seed0.pdf}}
%	\parbox{0.33\textwidth}{\includegraphics [width=0.33\textwidth] {Results/Gen-poly/Soft-exp/gen_softmax_exp_m-n-n_exp20-seed0.pdf}}
%	\parbox{0.33\textwidth}{\includegraphics [width=0.33\textwidth] {Results/Gen-poly/Soft-exp/gen_softmax_exp_m-onehalf-n-n_exp20-seed0.pdf}}
%	}
%\centerline
%{
%\parbox{0.33\textwidth}{\centering{(a)}}
%\parbox{0.33\textwidth}{\centering{(b)}}
%\parbox{0.33\textwidth}{\centering{(c)}}
%}
%\caption{ Softmax extension in exponential distribution  with general polytope and (a) $m = \lfloor 0.5 n \rfloor$ (b)  $m = n$}
%\label{Gen-Soft-Exp}
%\end{figure}

We generate the softmax objectives in the following way: first generate the $n$ eigenvalues $\vect{d} \in \mathbb{R}^n_{+}$, each randomly distributed in $[0, 1.5]$, and set $\vect{D} = diag(\vect{d})$. After generating a random unitary matrix $\vect{U}$, we set $\vect{L} = \vect{UDU^\top}$. One can verify that $\vect{L}$ is positive semidefinite and has eigenvalues as the entries of $\vect{d}$. We generate polytope constraints in the same form and same way as that for DR submodular quadratic and exponential functions, except for setting $b = 2\cdot \vect{1}^m$.
%In case of general polytope in uniform setting, we set the $b_m = -1$ whereas we set $b_m = -0.1$ for exponential distribution. Also, we set all the entries in the $m^{th}$ row of A to $-1$.
Function values returned by different solvers w.r.t. $n$ are shown in Figures \ref{Soft-Uni} and \ref{Soft-Exp}.

We can observe that our version of Frank-Wolfe performs at least as good as the two-phase Frank-Wolfe algorithm mentioned in~\cite{BianLevy17:Non-monotone-Continuous} for the down-closed polytope are generated using uniform distribution. In case of exponential distributions, our algorithms have comparable performance with the state-of-the-art algorithms.
%While considering the general polytope, our version of gradient ascent algorithm outperforms our version of Frank Wolfe in all cases.

\subsubsection{Offline Algorithm over General Polytopes}

Here, we consider the problem of revenue maximization on a (undirected) social network graph $G = (V, W)$, where $w_{ij} \in W$ represents the weight of the edge between vertex $i$ and vertex $j$. The goal is to offer for free or advertise a product to users so that the revenue increases through their ``word-of-mouth" effect on others. If one invests $x$ unit of cost on a user $i \in V$, the user $i$ becomes an advocate of the product (independently from other users) with probability $1- (1-p)^x$ where $p \in (0,1)$ is a parameter. Intuitively, it signifies that for investing a unit cost to $i$, we have an extra chance that the user $i$ becomes an advocate with probability $p$~\cite{Soma:2017}.

Let $S \subset V$ be a set of users who advocate for the product. Note that $S$ is random set. Then the revenue with respect to $S$ is defined as $\sum \limits_{i \in S} \sum \limits_{j \in V \setminus S} w_{ij}$. Let $f : \mathbb{Z}^E_+ \rightarrow \mathbb{R}$ be the expected revenue obtained in this model, that is
\begin{align*}
f(\vect{x}) &= \mathbb{E}_S \big[\sum \limits_{i \in S} \sum \limits_{j \in V \setminus S} w_{ij} \big] = \sum \limits_{i} \sum \limits_{j: i \not = j} w_{ij} (1-(1-p)^{x_i}) (1-p)^{x_j}
\end{align*}
It has been shown that $f$ is a non-monotone DR-submodular function~\cite{Soma:2017}. In our experiments, we used the Advogato network with $6.5$K users (vertices) and $61$K weighted relationship (edges). We set $p = 0.0001$. We imposed a minimum and a maximum investment constraint on the problem such that $0.25 \leq \sum_i x_i \leq 1$. This, in addition with $x_i \geq 0$ constitutes a general feasible polytope.

In Figure~\ref{real-world-data}(a), we show the performance of the Frank-Wolfe algorithm (as mentioned in Section~\ref{offline-dr-submod}) and commonly used Gradient Ascent algorithm. It is imperative to note that no performance guarantee is known for the Gradient Ascent algorithm for maximizing a non-monotone DR-submodular function over a general constraint polytope. We can clearly observe that the Frank-Wolfe algorithm performs at least as good as the commonly used Gradient Ascent algorithm.

\begin{figure}[h]
\vspace{-1.5cm}
\centerline {
	\parbox{0.33\textwidth}{\includegraphics [width=0.33\textwidth] {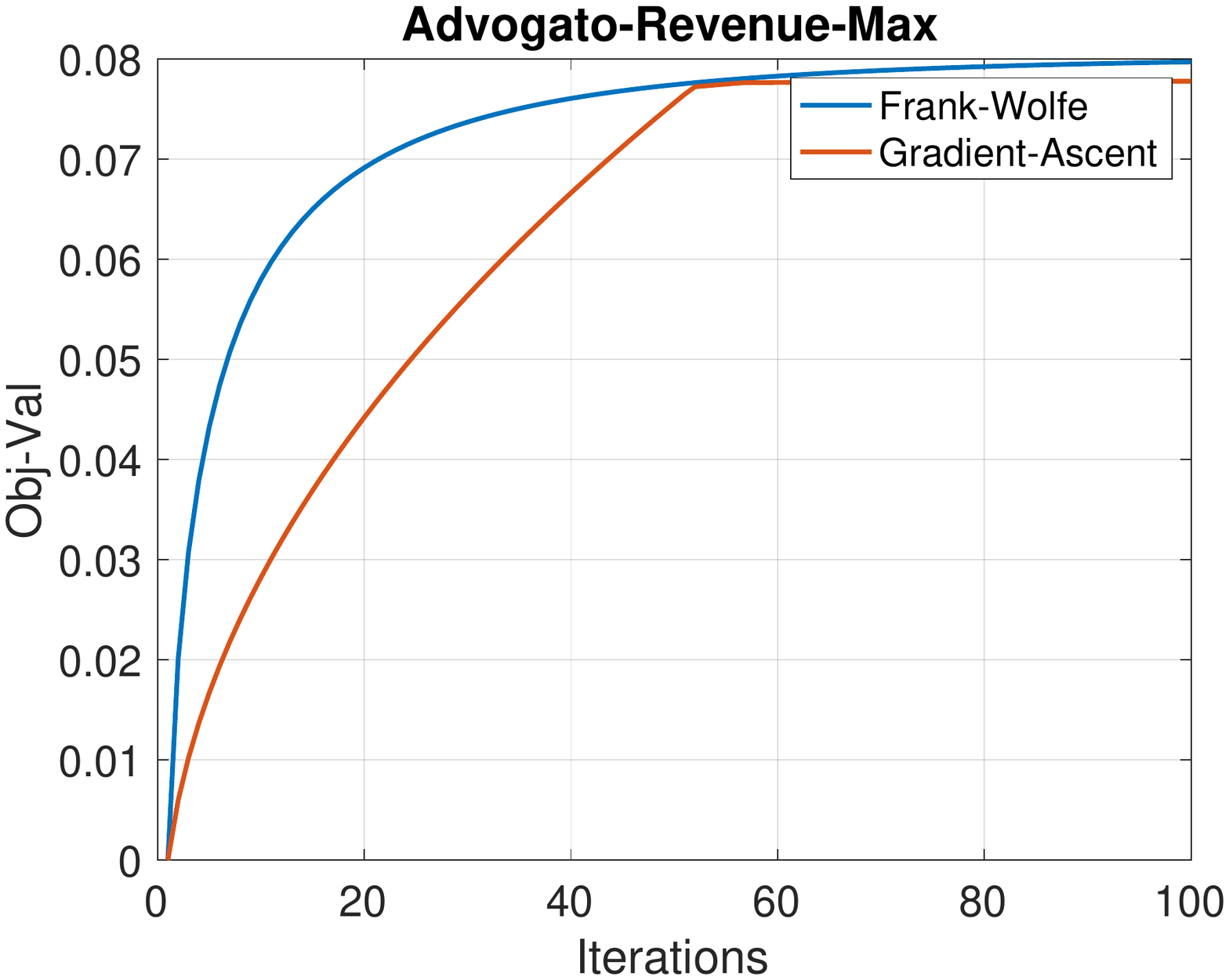}}
	\parbox{0.33\textwidth}{\includegraphics [width=0.33\textwidth] {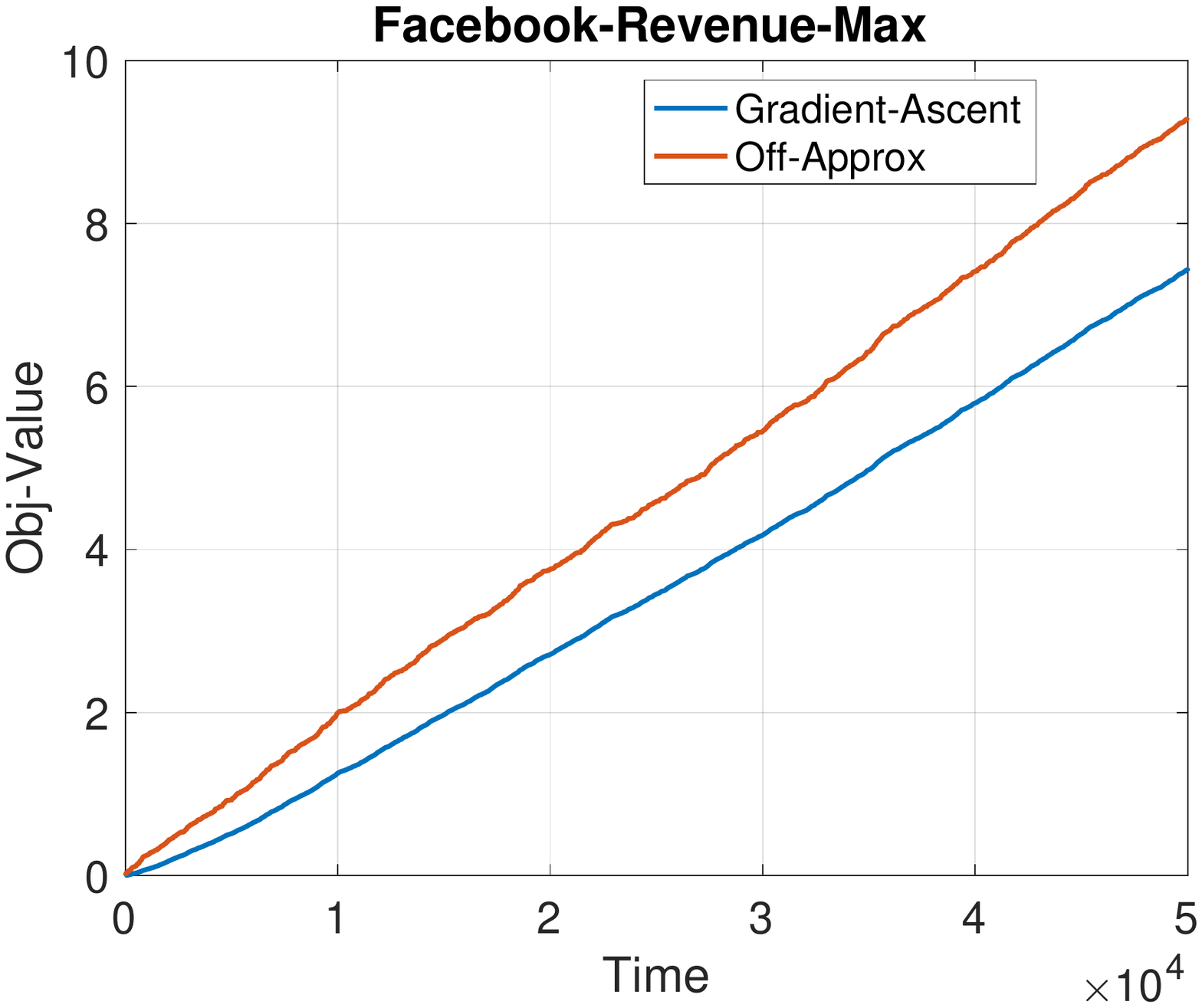}}
	\parbox{0.33\textwidth}{\includegraphics [width=0.33\textwidth] {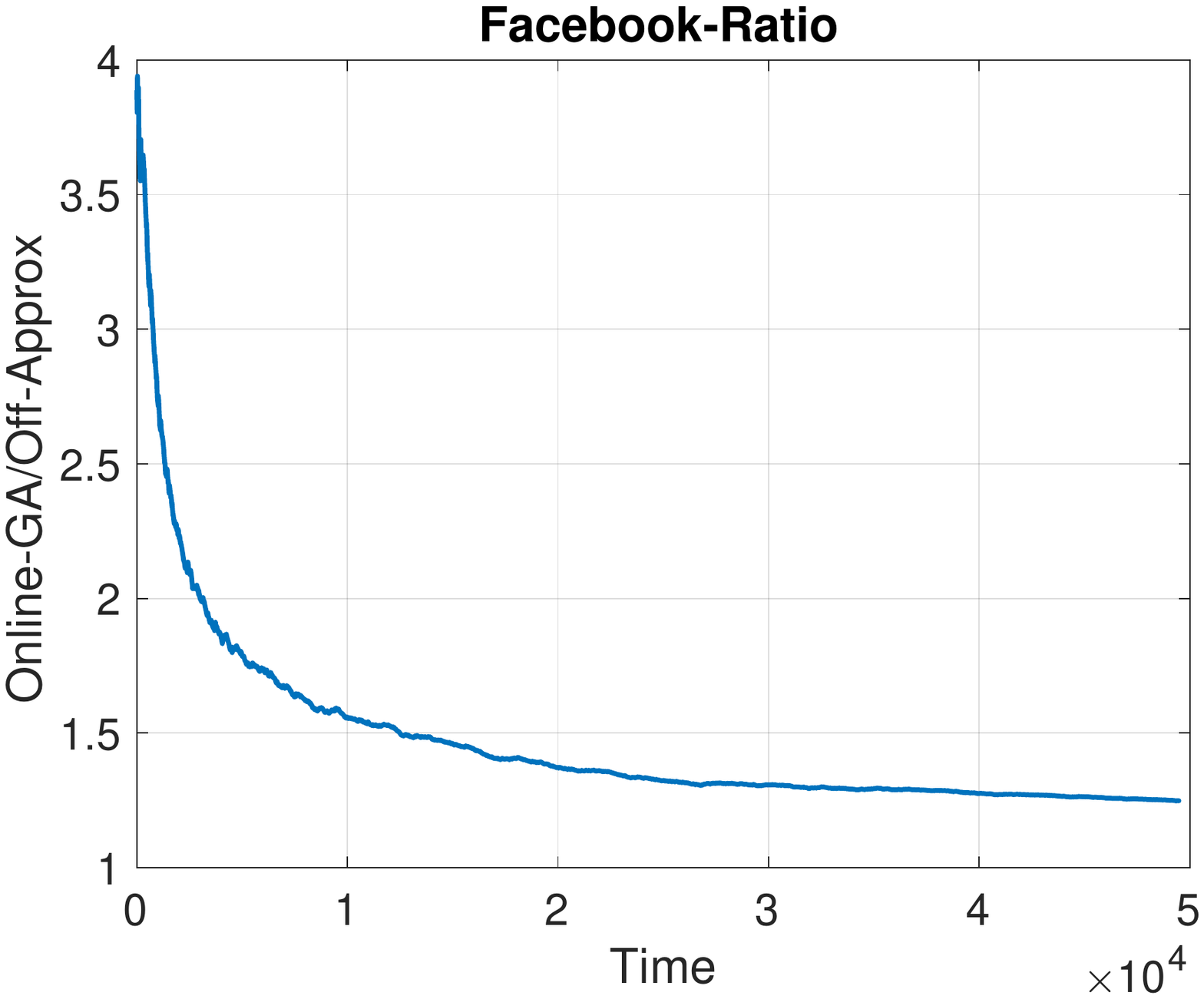}}
	}
\vspace{-1.0cm}
\centerline
{
\parbox{0.33\textwidth}{\centering{(a)}}
\parbox{0.33\textwidth}{\centering{(a)}}
\parbox{0.33\textwidth}{\centering{(b)}}
}
\caption{(a) Offline Revenue Maximization on Advogato dataset, (b) Online Revenue Maximization for Facebook dataset for a down-closed polytope and (c) Ratio of performance of Online Gradient Ascent over Offline approximation for Facebook dataset for a down-closed polytope}
\label{real-world-data}
\end{figure}

\subsubsection{Online Algorithm over Down-closed Polytopes}

In this subsection, we consider the online variant of the revenue maximization on a (undirected) social network where at time $t$ the weight of an edge is given $w^t_{ij} \in \{0,1\}$. The experiments are performed on the Facebook dataset that contains $64$K users (vertices) and $1$M relationships (edges). We choose the number of time steps to be $T=50,000$. At each time $t \in {1, \ldots, T}$, we randomly uniformly select $2000$ vertices $V^t \subset V$, independently of $V^1,\ldots,V^{t-1}$, and construct a batch $B_t$ with edge-weights $w^t_{ij} = 1$ if and only if $i,j \in V^t$ and edge $(i,j)$ exists in the Facebook dataset. In case if $i$ or $j$ do not belong to $V^t$, $w^t_{ij} = 0$.

We again set $p = 0.0001$ and impose a maximum investment constraint on the problem such that $\sum \limits_{i \in V} x_i \leq 1$. This, in addition to $x_i \geq 0, \forall i \in V$ constitutes a down-closed feasible polytope.

For comparison purposes, we chose the (offline) Frank-Wolfe algorithm that is shown to be $\frac{1}{e}$-approximation for maximizing non-monotone DR-submodular function over down-closed polytopes~\cite{BianLevy17:Non-monotone-Continuous}. Using this algorithm, we first computed  $x^*$ such that $x^*$ approximately maximizes $\sum_t F^t(x)$ and then computed $\sum_{t \leq t'} F^t(x^*)$ for every $t' \in {1,\ldots, T}$. In Figure~\ref{real-world-data}(b), we show how the function $\sum_{t \leq t'} F^t(x)$ evolves with time $t'$ for the Online Gradient Ascent algorithm (as mentioned in Section~\ref{online-dr-submod}) in comparison to $\sum_{t \leq t'} F^t(x^*)$. In Figure~\ref{real-world-data}(c), we show the ratio of between the objective value achieved by the Online Ascent algorithm and $\sum_{t \leq t'} F^t(x^*)$. The gradual reduction in this ratio (over time) conforms with the fact that the additive term in the theoretical guarantee reduces with time.

\paragraph{Projection on $\cal K$.}

The typical implementation of the projection operator $\textrm{Proj}_{\cal k}(\vect x)$  would consist of solving the quadratic program, where we want to minimize $\| \vect{x} - \vect{y} \|$ under the constraint $\vect{y} \in \cal K$.
However using the particular structure of our convex space $\cal K$ a more efficient projection operator is possible.  Recall that $\cal K$ is defined as the set of all points $\vect y$ with non negative coordinates and $\sum_{i=1}^n \vect{y}_i \leq 1$.  Without loss of generality the entries of $\vect{x}$ are sorted in non increasing order $\vect{x}_1 \geq \ldots \vect{x}_n$.  If $\vect{x}\not\in\cal K$, then its projection is a vector $\vect{y}$ of the form
\[
			\vect{y} = (\vect{x}_1 - \delta_j, \ldots , \vect{x}_j - \delta_j, 0, \ldots, 0),
\]
for some index $j$ and $\delta_j = (\vect{x}_1 + \ldots \vect{x}_j - 1) / j$.  In the degenerate case when all entries of $\vect x$ are non positive, the projection is the vector $\vect 0$.
The distance between $\vect{x}$ and its projection $\vect{y}$ is then
\[
			\|\vect{x} - \vect{y}\|^2 = \delta^2 \cdot j + \vect{x}_{j+1}^2 + \ldots + \vect{x}_{n}^2.
\]
By optimality of the projection, we have that $j$ is the maximal index satisfying $\delta_j \geq \vect{x}_j$.
Hence the projection can be computed in time $O(n \log n)$, by first sorting the entries of $\vect{x}$ in time $O(n\log n)$ and then by iterating over $j$, maintaining in constant time the sum $\vect{x}_1 + \ldots + \vect{x}_j$.  However we observed better performance of another projection algorithm with complexity $O(n^2)$ which exploits better the possibilities of MATLAB.  It consists of an iterative procedure. While $\vect{x} \not\in \cal K$ we project $\vect{x}$ to the positive sub-space (i.e.\ set all its negative entries to 0) and then remove $\delta$ from all non zero entries, where $\delta$ is the average of all non zero entries in $\vect{x}$.  While this procedure can iterate $n$ times in the worst case, in practice it often iterates only a constant number of times.

\section{Conclusion}

In this paper, we have provided performance guarantees for the problems of 
maximizing non-monotone submodular/DR-submodular functions over convex sets in 
offline and online environments. These results are completed by experiments in different contexts.
Moreover, the results give raise to the question of designing online algorithms for non-monotone 
DR-submodular maximization over a general convex set.  
Characterizing necessary and sufficient regularity conditions/structures that enable 
efficient algorithm with approximation and regret guarantees is an interesting direction to pursue. 

%\bibliographystyle{plainnat}
%\bibliography{submodular} 

\begin{thebibliography}{35}
\providecommand{\natexlab}[1]{#1}
\providecommand{\url}[1]{\texttt{#1}}
\expandafter\ifx\csname urlstyle\endcsname\relax
  \providecommand{\doi}[1]{doi: #1}\else
  \providecommand{\doi}{doi: \begingroup \urlstyle{rm}\Url}\fi

\bibitem[Bach(2016)]{Bach16:Submodular-functions:}
Francis Bach.
\newblock Submodular functions: from discrete to continuous domains.
\newblock \emph{Mathematical Programming}, pages 1--41, 2016.

\bibitem[Bach et~al.(2013)]{Bachothers13:Learning-with}
Francis Bach et~al.
\newblock Learning with submodular functions: A convex optimization
  perspective.
\newblock \emph{Foundations and Trends{\textregistered} in Machine Learning},
  6\penalty0 (2-3):\penalty0 145--373, 2013.

\bibitem[Bian et~al.(2018)Bian, Buhmann, and
  Krause]{BianBuhmann18:Optimal-DR-Submodular}
An~Bian, Joachim~M Buhmann, and Andreas Krause.
\newblock Optimal {DR}-submodular maximization and applications to provable
  mean field inference.
\newblock In \emph{Neural Information Processing Systems (NIPS)}, 2018.

\bibitem[Bian et~al.(2017{\natexlab{a}})Bian, Levy, Krause, and
  Buhmann]{BianLevy17:Non-monotone-Continuous}
Andrew~An Bian, Kfir Levy, Andreas Krause, and Joachim~M. Buhmann.
\newblock Non-monotone continuous {DR}-submodular maximization: Structure and
  algorithms.
\newblock In \emph{Neural Information Processing Systems (NIPS)},
  2017{\natexlab{a}}.

\bibitem[Bian et~al.(2017{\natexlab{b}})Bian, Mirzasoleiman, Buhmann, and
  Krause]{BianMirzasoleiman17:Guaranteed-Non-convex}
Andrew~An Bian, Baharan Mirzasoleiman, Joachim Buhmann, and Andreas Krause.
\newblock Guaranteed non-convex optimization: Submodular maximization over
  continuous domains.
\newblock In \emph{Artificial Intelligence and Statistics (AISTATS)}, pages
  111--120, 2017{\natexlab{b}}.

\bibitem[Buchbinder and
  Feldman(2018)]{BuchbinderFeldman18:Deterministic-algorithms}
Niv Buchbinder and Moran Feldman.
\newblock Deterministic algorithms for submodular maximization problems.
\newblock \emph{ACM Transactions on Algorithms (TALG)}, 14\penalty0
  (3):\penalty0 32, 2018.

\bibitem[Buchbinder et~al.(2015)Buchbinder, Feldman, Seffi, and
  Schwartz]{BuchbinderFeldman15:A-tight-linear}
Niv Buchbinder, Moran Feldman, Joseph Seffi, and Roy Schwartz.
\newblock A tight linear time (1/2)-approximation for unconstrained submodular
  maximization.
\newblock \emph{SIAM Journal on Computing}, 44\penalty0 (5):\penalty0
  1384--1402, 2015.

\bibitem[Calinescu et~al.(2011)Calinescu, Chekuri, P{\'a}l, and
  Vondr{\'a}k]{CalinescuChekuri11:Maximizing-a-monotone}
Gruia Calinescu, Chandra Chekuri, Martin P{\'a}l, and Jan Vondr{\'a}k.
\newblock Maximizing a monotone submodular function subject to a matroid
  constraint.
\newblock \emph{SIAM Journal on Computing}, 40\penalty0 (6):\penalty0
  1740--1766, 2011.

\bibitem[Chekuri et~al.(2015)Chekuri, Jayram, and
  Vondr{\'a}k]{ChekuriJayram15:On-multiplicative-weight}
Chandra Chekuri, TS~Jayram, and Jan Vondr{\'a}k.
\newblock On multiplicative weight updates for concave and submodular function
  maximization.
\newblock In \emph{Conference on Innovations in Theoretical Computer Science
  (ITCS)}, pages 201--210, 2015.

\bibitem[Chen et~al.(2018)Chen, Hassani, and
  Karbasi]{ChenHassani18:Online-continuous}
Lin Chen, Hamed Hassani, and Amin Karbasi.
\newblock Online continuous submodular maximization.
\newblock In \emph{Proc. 21st International Conference on Artificial
  Intelligence and Statistics (AISTAT)}, 2018.

\bibitem[Djolonga and Krause(2014)]{DjolongaKrause14:From-MAP-to-marginals:}
Josip Djolonga and Andreas Krause.
\newblock From map to marginals: Variational inference in bayesian submodular
  models.
\newblock In \emph{Advances in Neural Information Processing Systems (NIPS)},
  pages 244--252, 2014.

\bibitem[Djolonga et~al.(2016)Djolonga, Tschiatschek, and
  Krause]{DjolongaTschiatschek16:Variational-inference}
Josip Djolonga, Sebastian Tschiatschek, and Andreas Krause.
\newblock Variational inference in mixed probabilistic submodular models.
\newblock In \emph{Advances in Neural Information Processing Systems (NIPS)},
  pages 1759--1767, 2016.

\bibitem[Elenberg et~al.(2017)Elenberg, Dimakis, Feldman, and
  Karbasi]{ElenbergDimakis17:Streaming-weak}
Ethan Elenberg, Alexandros~G Dimakis, Moran Feldman, and Amin Karbasi.
\newblock Streaming weak submodularity: Interpreting neural networks on the
  fly.
\newblock In \emph{Advances in Neural Information Processing Systems (NIPS)},
  pages 4044--4054, 2017.

\bibitem[Elenberg et~al.(2018)Elenberg, Khanna, Dimakis, Negahban,
  et~al.]{ElenbergKhanna18:Restricted-strong}
Ethan~R Elenberg, Rajiv Khanna, Alexandros~G Dimakis, Sahand Negahban, et~al.
\newblock Restricted strong convexity implies weak submodularity.
\newblock \emph{The Annals of Statistics}, 46\penalty0 (6B):\penalty0
  3539--3568, 2018.

\bibitem[Feldman et~al.(2011)Feldman, Naor, and
  Schwartz]{FeldmanNaor11:A-unified-continuous}
Moran Feldman, Joseph Naor, and Roy Schwartz.
\newblock A unified continuous greedy algorithm for submodular maximization.
\newblock In \emph{Proc. 52nd Symposium on Foundations of Computer Science
  (FOCS)}, pages 570--579, 2011.

\bibitem[Fujishige(2005)]{Fujishige05:Submodular-functions}
Satoru Fujishige.
\newblock \emph{Submodular functions and optimization}, volume~58.
\newblock Elsevier, 2005.

\bibitem[Gillenwater et~al.(2012)Gillenwater, Kulesza, and
  Taskar]{GillenwaterKulesza12:Near-optimal-map-inference}
Jennifer Gillenwater, Alex Kulesza, and Ben Taskar.
\newblock Near-optimal map inference for determinantal point processes.
\newblock In \emph{Advances in Neural Information Processing Systems (NIPS)},
  pages 2735--2743, 2012.

\bibitem[Golovin and Krause(2011)]{GolovinKrause11:Adaptive-submodularity:}
Daniel Golovin and Andreas Krause.
\newblock Adaptive submodularity: Theory and applications in active learning
  and stochastic optimization.
\newblock \emph{Journal of Artificial Intelligence Research}, 42:\penalty0
  427--486, 2011.

\bibitem[Gomez~Rodriguez et~al.(2010)Gomez~Rodriguez, Leskovec, and
  Krause]{Gomez-RodriguezLeskovec10:Inferring-networks}
Manuel Gomez~Rodriguez, Jure Leskovec, and Andreas Krause.
\newblock Inferring networks of diffusion and influence.
\newblock In \emph{International conference on Knowledge discovery and data
  mining (SIGKDD)}, pages 1019--1028, 2010.

\bibitem[Guillory and Bilmes(2011)]{GuilloryBilmes11:Simultaneous-Learning}
Andrew Guillory and Jeff~A Bilmes.
\newblock Simultaneous learning and covering with adversarial noise.
\newblock In \emph{International Conference on Machine Learning (ICML)},
  volume~11, pages 369--376, 2011.

\bibitem[Hassani et~al.(2017)Hassani, Soltanolkotabi, and
  Karbasi]{HassaniSoltanolkotabi17:Gradient-methods}
Hamed Hassani, Mahdi Soltanolkotabi, and Amin Karbasi.
\newblock Gradient methods for submodular maximization.
\newblock In \emph{Advances in Neural Information Processing Systems}, pages
  5841--5851, 2017.

\bibitem[Hazan(2016)]{Hazanothers16:Introduction-to-online}
Elad Hazan.
\newblock Introduction to online convex optimization.
\newblock \emph{Foundations and Trends{\textregistered} in Optimization},
  2\penalty0 (3-4):\penalty0 157--325, 2016.

\bibitem[Ito and Fujimaki(2016)]{ItoFujimaki16:Large-scale-price}
Shinji Ito and Ryohei Fujimaki.
\newblock Large-scale price optimization via network flow.
\newblock In \emph{Advances in Neural Information Processing Systems (NIPS)},
  pages 3855--3863, 2016.

\bibitem[Iwata et~al.(2001)Iwata, Fleischer, and
  Fujishige]{IwataFleischer01:A-combinatorial-strongly}
Satoru Iwata, Lisa Fleischer, and Satoru Fujishige.
\newblock A combinatorial strongly polynomial algorithm for minimizing
  submodular functions.
\newblock \emph{Journal of the ACM}, 48\penalty0 (4):\penalty0 761--777, 2001.

\bibitem[Kempe et~al.(2003)Kempe, Kleinberg, and
  Tardos]{KempeKleinberg03:Maximizing-the-spread}
David Kempe, Jon Kleinberg, and {\'E}va Tardos.
\newblock Maximizing the spread of influence through a social network.
\newblock In \emph{International conference on Knowledge discovery and data
  mining (SIGKDD)}, pages 137--146, 2003.

\bibitem[Kulesza et~al.(2012)Kulesza, Taskar,
  et~al.]{KuleszaTaskar12:Determinantal-point}
Alex Kulesza, Ben Taskar, et~al.
\newblock Determinantal point processes for machine learning.
\newblock \emph{Foundations and Trends{\textregistered} in Machine Learning},
  5\penalty0 (2--3):\penalty0 123--286, 2012.

\bibitem[Lin and Bilmes(2011)]{LinBilmes11:A-class-of-submodular}
Hui Lin and Jeff Bilmes.
\newblock A class of submodular functions for document summarization.
\newblock In \emph{Association for Computational Linguistics (ACL)}, pages
  510--520, 2011.

\bibitem[Nemhauser et~al.(1978)Nemhauser, Wolsey, and
  Fisher]{NemhauserWolsey78:An-analysis-of-approximations}
George~L Nemhauser, Laurence~A Wolsey, and Marshall~L Fisher.
\newblock An analysis of approximations for maximizing submodular set
  functions.
\newblock \emph{Mathematical programming}, 14\penalty0 (1):\penalty0 265--294,
  1978.

\bibitem[Niazadeh et~al.(2018)Niazadeh, Roughgarden, and
  Wang]{NiazadehRoughgarden18:Optimal-Algorithms}
Rad Niazadeh, Tim Roughgarden, and Joshua~R Wang.
\newblock Optimal algorithms for continuous non-monotone submodular and
  dr-submodular maximization.
\newblock In \emph{Neural Information Processing Systems (NIPS)}, 2018.

\bibitem[Roughgarden and Wang(2018)]{RoughgardenWang18:An-Optimal-Algorithm}
Tim Roughgarden and Joshua~R Wang.
\newblock An optimal algorithm for online unconstrained submodular
  maximization.
\newblock In \emph{Conference on Learning Theory (COLT)}, 2018.

\bibitem[Schrijver(2000)]{Schrijver00:A-combinatorial-algorithm}
Alexander Schrijver.
\newblock A combinatorial algorithm minimizing submodular functions in strongly
  polynomial time.
\newblock \emph{Journal of Combinatorial Theory, Series B}, 80\penalty0
  (2):\penalty0 346--355, 2000.

\bibitem[Singla et~al.(2014)Singla, Bogunovic, Bart{\'o}k, Karbasi, and
  Krause]{SinglaBogunovic14:Near-Optimally-Teaching}
Adish Singla, Ilija Bogunovic, G{\'a}bor Bart{\'o}k, Amin Karbasi, and Andreas
  Krause.
\newblock Near-optimally teaching the crowd to classify.
\newblock In \emph{International conference on Machine Learning ICML}, pages
  154--162, 2014.

\bibitem[Soma and Yoshida(2017)]{Soma:2017}
Tasuku Soma and Yuichi Yoshida.
\newblock Non-monotone dr-submodular function maximization.
\newblock In \emph{AAAI conference on Artificial Intelligence (AAAI)}, pages
  898--904, 2017.

\bibitem[Staib and Jegelka(2017)]{StaibJegelka17:Robust-Budget}
Matthew Staib and Stefanie Jegelka.
\newblock Robust budget allocation via continuous submodular functions.
\newblock In \emph{International Conference on Machine Learning (ICML)}, pages
  3230--3240, 2017.

\bibitem[Vondr{\'a}k(2013)]{Vondrak13:Symmetry-and-approximability}
Jan Vondr{\'a}k.
\newblock Symmetry and approximability of submodular maximization problems.
\newblock \emph{SIAM Journal on Computing}, 42\penalty0 (1):\penalty0 265--304,
  2013.

\end{thebibliography}

\newpage
\appendix
\section{Properties of DR-submodularity}
We provide the proofs of the properties of DR-submodular functions mentioned in Section~\ref{sec:pre}.

\setcounter{lemma}{0}

\begin{lemma}[\cite{HassaniSoltanolkotabi17:Gradient-methods}]
For every $\vect{x}, \vect{y} \in \mathcal{K}$ and any DR-submodular function $F: [0,1]^{n} \rightarrow \mathbb{R}^{+}$,
it holds that
$$ \langle \nabla F(\vect{x}), \vect{y} - \vect{x}\rangle \geq F(\vect{x} \vee \vect{y}) + F(\vect{x} \wedge \vect{y}) - 2F(\vect{x}). $$
\end{lemma}
\begin{proof}
For any vectors $\vect{x} \leq \vect{z}$, using Inequality~(\ref{def:DR-sub-diff}), we have
\begin{align*}
F(\vect{z}) - F(\vect{x})
	&= \int_{0}^{1} \bigl \langle \vect{z} - \vect{x}, \nabla F\bigl(\vect{x} + t(\vect{z} - \vect{x}) \bigr) \bigr\rangle dt \\
&\leq \int_{0}^{1} \bigl \langle \vect{z} - \vect{x}, \nabla F(\vect{x}) \bigr\rangle dt
= \bigl \langle \vect{z} - \vect{x}, \nabla F(\vect{x}) \bigr\rangle.
\end{align*}
Therefore,
\begin{align}	\label{ineq:concave-dr1}
F(\vect{x} \vee \vect{y}) - F(\vect{x}) \leq \bigl \langle \vect{x} \vee \vect{y} - \vect{x}, \nabla F(\vect{x}) \bigr\rangle.
\end{align}

Similarly for vectors $\vect{x} \leq \vect{z}$, we have
\begin{align*}
F(\vect{z}) - F(\vect{x}) &= \int_{0}^{1} \bigl \langle \vect{z} - \vect{x}, \nabla F\bigl(\vect{x} + t(\vect{z} - \vect{x}) \bigr) \bigr\rangle dt \\
	&\geq \int_0^1 \langle \vect{z} - \vect{x}, \nabla F(\vect{z}) \rangle dt
	= \langle \vect{z} - \vect{x}, \nabla F(\vect{z}) \rangle.
\end{align*}
Therefore,
\begin{align}	\label{ineq:concave-dr2}
F(\vect{x} \wedge \vect{y}) - F(\vect{x}) \leq \bigl \langle \vect{x} \wedge \vect{y} - \vect{x}, \nabla F(\vect{x}) \bigr\rangle
\end{align}
Summing (\ref{ineq:concave-dr1}) and (\ref{ineq:concave-dr2}) and using the fact
$(\vect{x} \vee \vect{y}) +  (\vect{x} \wedge \vect{y}) = \vect{x} + \vect{y}$, we obtain
$$
F(\vect{x} \vee \vect{y}) + F(\vect{x} \wedge \vect{y}) - 2F(\vect{x})
	\leq \bigl \langle \vect{y} - \vect{x}, \nabla F(\vect{x}) \bigr\rangle.
$$
\end{proof}

\begin{lemma}[\cite{BianLevy17:Non-monotone-Continuous}]
For any DR-submodular function $F$ and for all $\vect{x}, \vect{y}, \vect{z} \in \mathcal{K}$ it holds that
$$F(\vect{x} \vee \vect{y}) + F(\vect{x} \wedge \vect{y}) + F(\vect{z}^* \vee \vect{z}) + F(\vect{z}^* \wedge \vect{z}) \geq F(\vect{y}),$$ 
where $\vect{z}^* = (\vect{x} \vee \vect{y}) - \vect{x}$.
\end{lemma}

\begin{proof}
First, we claim the following two inequalities:
\begin{align}
F(\vect{x} \vee \vect{y}) + F(\vect{z} \vee \vect{z}^*) &\geq F(\vect{z}^*) + F((\vect{x}+\vect{z}) \vee \vect{y})	\label{online-l1-eq-1}  \\
F(\vect{z}^*) + F(\vect{x} \wedge \vect{y}) &\geq F(\vect{y}) + F(\vect{0}).  \label{online-l1-eq-2}
\end{align}
Assuming (\ref{online-l1-eq-1}) and (\ref{online-l1-eq-2}) holds, we get:
\begin{align*}
F(\vect{x} \vee \vect{y}) + F(\vect{z} \vee \vect{z}^*) + F(\vect{x} \wedge \vect{y}) + F(\vect{z} \vee \vect{z}^*)  &\geq  F(\vect{y}) + F(0)  + F((\vect{x}+\vect{z}) \vee \vect{y}) +  F(\vect{z} \vee \vect{z}^*) \\
&\geq F(\vect{y})
\end{align*}
and the lemma follows. In the remaining, we prove the above two inequalities.

First, we establish the following identity.
\begin{align}	\label{eq:identity}
\vect{x} \vee \vect{y} - \vect{z}^* &= (\vect{x}+\vect{z}) \vee \vect{y} - \vect{z} \vee \vect{z}^*
\end{align}
For this purpose, we will show that both the RHS and LHS of (\ref{eq:identity}) are equal to $\vect{x}$. For the LHS we can write $\vect{x} \vee \vect{y} - \vect{z}^* = \vect{x} \vee \vect{y} - (\vect{x} \vee \vect{y} - \vect{x}) = \vect{x}$. For the RHS, let us consider any coordinate $i \in [n]$, and show that the following expression equals $x_i$:
$$
(x_i + z_i) \vee y_i - z_i \vee z^*_i = (x_i + z_i) \vee y_i - ((x_i + z_i) - x_i) \vee ((x_i \vee y_i) - x_i).
$$
\begin{description}
	\item[Case $(x_i + z_i) \geq y_i$.] So $(x_i + z_i)$ is larger than both $y_{i}$ and $x_{i}$. Therefore,
	\begin{align*}
		(x_i + z_i) \vee y_i - ((x_i + z_i) - x_i) \vee ((x_i \vee y_i) - x_i)
		= (x_i + z_i) - ((x_i + z_i) - x_i) = x_{i}.
	\end{align*}
	\item[Case $(x_i + z_i) < y_i$.] So $(x_i \vee y_i) = y_{i}$. Therefore,
	\begin{align*}
		(x_i + z_i) \vee y_i - ((x_i + z_i) - x_i) \vee ((x_i \vee y_i) - x_i)
		= y_{i} - ((x_i \vee y_i) - x_i) = x_{i}.
	\end{align*}
\end{description}
Hence, the RHS of (\ref{eq:identity}) is equal to $\vect{x}$. So the identity (\ref{eq:identity}) holds.

We are now proving Inequality (\ref{online-l1-eq-1}), i.e.,
$$
F(\vect{x} \vee \vect{y}) -F(\vect{z}^*) \geq  F((\vect{x}+\vect{z}) \vee \vect{y}) -  F(\vect{z} \vee \vect{z}^*).
$$
The above inequality holds due to (\ref{eq:identity}), the fact $\vect{z}^* \leq \vect{z} \vee \vect{z}^*$ and
the diminishing return property of $F$.

Now we prove Inequality (\ref{online-l1-eq-2}), i.e.,
$$
F(\vect{z}^*) - F(\vect{0}) \geq F(\vect{y}) - F(\vect{x} \wedge \vect{y}).
$$
The above inequality holds by the diminishing return property and
$$
\vect{y} - \vect{x} \wedge \vect{y} = \vect{x} \vee \vect{y} - \vect{x} = \vect{z}^* - \vect{0}
\qquad \text{~and~} \qquad
\vect{0} \leq \vect{x} \wedge \vect{y}.
$$
\end{proof}

\end{document}